\documentclass{article}
\usepackage{microtype}
\usepackage{graphicx}
\usepackage{subcaption}
\usepackage{booktabs} 
\usepackage{amsmath} 
\usepackage{graphicx} 
\usepackage{color}
\usepackage{soul}
\renewcommand{\hl}[1]{#1}
\usepackage{hyperref}
\usepackage{caption}

\newcommand{\brwonian}{\mathbf{w}}

\usepackage{mdframed}

\usepackage[accepted]{icml2026}

\usepackage{amsmath}
\usepackage{amssymb}
\usepackage{mathtools}
\usepackage{amsthm}

\usepackage{graphics}
\usepackage[utf8]{inputenc} 
\usepackage[T1]{fontenc}    
\usepackage{hyperref}       
\usepackage{url}            
\usepackage{booktabs}       
\usepackage{amsfonts}       
\usepackage{nicefrac}       
\usepackage{microtype}      
\usepackage{xcolor}         

\usepackage{graphicx}
\usepackage{soul}
\usepackage[ruled,vlined,algo2e]{algorithm2e}
\usepackage{booktabs}
\usepackage{boldline}
\usepackage{siunitx}
\usepackage{wrapfig}
\usepackage{tabularray}
\usepackage{multirow}
\usepackage{arydshln}

\usepackage[capitalize,noabbrev]{cleveref}

\theoremstyle{plain}
\newtheorem{theorem}{Theorem}[section]
\newtheorem{proposition}[theorem]{Proposition}
\newtheorem{lemma}[theorem]{Lemma}
\newtheorem{corollary}[theorem]{Corollary}
\theoremstyle{definition}
\newtheorem{definition}[theorem]{Definition}
\newtheorem{assumption}[theorem]{Assumption}
\theoremstyle{remark}

\usepackage[textsize=tiny]{todonotes}
\usepackage{listings}

\icmltitlerunning{BSDE-based Diffusion}

\begin{document}
\lstnewenvironment{pythoncode}[1][]
{\lstset{language=Python, basicstyle=\ttfamily\small, keywordstyle=\bfseries, stringstyle=\itshape, showstringspaces=false, frame=single, #1}}
{}

\twocolumn[
  \icmltitle{Backward SDEs-based Diffusion for Physics-Constrained Generation}



  \icmlsetsymbol{equal}{*}

  \begin{icmlauthorlist}
    \icmlauthor{Zihao Wang}{yyy}
  \end{icmlauthorlist}

  \icmlaffiliation{yyy}{\url{https://laplace.center}, Department of Computer Science, University of Tennessee, Chattanooga, US}

  \icmlcorrespondingauthor{Zihao Wang}{zihao.wang@tennessee.edu; zihao.wang@ieee.org}

  \icmlkeywords{Machine Learning, ICML}

  \vskip 0.3in
]



\printAffiliationsAndNotice{Laplace Lab}  

\begin{abstract}
Pretrained score-based diffusion models provide strong unconditional priors, yet enforcing measurement or physics consistency in inverse problems is often handled by heuristic guidance, intermittent projections, or task-specific conditional training, with limited guarantees of feasibility at the end of inference. We propose terminal-conditioned inversion for score-based SDE priors. Given a frozen Score-SDE prior and a task-defined terminal feasibility specification, we construct an associated backward stochastic differential equation whose adapted solution defines a principled inverse map from the terminal requirement to a prior state at a chosen noise level. Under standard regularity conditions, we establish existence and uniqueness of the adapted solution and obtain terminal consistency by construction. We further develop a practical neural BSDE solver that composes arbitrary pretrained diffusion priors with domain constraints without modifying the score-defined coefficients, producing an anchored prior state that enables neighborhood sampling for uncertainty characterization. Experiments on toy datasets validate stable terminal-conditioned inversion and distributionally consistent neighborhood sampling. As a real-world case study, we apply the framework to sparse-view CT reconstruction and achieve improved reconstruction quality over representative training-free baselines while satisfying strict measurement feasibility under the prescribed terminal specification. 
Project is available in: \href{https://laplacelab.github.io/BSDEDiffusion/}{https://laplace.center/icmlbsdeI/}
\end{abstract}

\section{Introduction}
\label{sec:intro}

Many practical and scientific tasks require recovering an unknown signal $x$ from indirect and noisy observations
\begin{equation}
y = \mathcal{A}(x) + \eta,
\label{eq:meas}
\end{equation}
where $\mathcal{A}$ is a known or partially known forward operator and $\eta$ denotes noise.
The forward model alone rarely identifies $x$ uniquely: multiple candidates can explain the same $y$.
In high-stakes settings, the goal is not merely to produce a perceptually plausible sample, but to output
solutions that are physically consistent with the observation at the end of inference, for example
\begin{equation}
x \in \mathcal{S}(y) \triangleq \{x:\ \mathcal{A}(x)\approx y\}.
\end{equation}
This exposes a persistent tension.
Physics imposes a rigid terminal requirement, while unsupervised generative priors represent uncertainty through
stochasticity and high-entropy support.
Bridging physical determinism with generative uncertainty is the core challenge.

\begin{table*}[t]
\centering
\small
\setlength{\tabcolsep}{6pt}
\resizebox{\textwidth}{!}{
\begin{tabular}{lccccc}
\hline
\textbf{Framework} & \textbf{Prior} & \textbf{Constraint form} & \textbf{Path cond.?} & \textbf{Terminal consistency} & \textbf{Tuning burden}\\
\hline
Variational, TV & explicit reg. / plug-in & soft, penalty & No & sometimes & medium \\
GAN inversion & GAN generator & soft, physics loss & No & no guarantee & high \\
Diffusion posterior sampling & score & soft & Approx. & no guarantee & high \\
Projection-in-the-loop/PnP & domain pretrained score & intermittent projection & No & often improves, not guaranteed & medium\\
Manifold-constrained diffusion & score & soft or hybrid & Approx. & not guaranteed & medium \\
Diffusion and Schr\"odinger bridges & score & path measure & Yes, path value & yes & high \\
\hline
\textbf{Ours: BSDE--diffusion} & general pretrained score & \textbf{hard terminal map} & \textbf{Yes, terminal value} & \textbf{by design} & \textbf{lower}\\
\hline
\end{tabular}}
\caption{
Comparison of physics-constrained generative frameworks:
Variational~\citep{rudin1992rof,sidky2008tvct,diptv,tvldct};
GAN inversion~\citep{bora2017csgm,shah2018ganpriors,chan2022gleangenerativelatentbank};
DPS and posterior sampling~\citep{chung2022dps,li2024ct_dps_jmi,pmlr-v267-alkhouri25b,pmlr-v267-achituve25a,pmlr-v267-ekstrom-kelvinius25b,pmlr-v235-peng24h};
PnP, RED, diffusion PnP~\citep{yuyangpnp,venkatakrishnan2013pnp,romano2017red,wang2022ddnm,zhu2023diffpir};
manifold constrained~\citep{chung2022mcg,tip,chen2024flowmatchinggeneralgeometries};
bridges and matching~\citep{debortoli2021dsb,heng2021diffusionbridge,wang2025irbridge,pmlr-v235-gushchin24a,pmlr-v267-gushchin25b,pmlr-v267-ksenofontov25a,pmlr-v267-alkhouri25b}.
}
\label{tab:comparison}
\end{table*}

Diffusion and score-based models~\citep{ddpm,song2021scorebased} provide remarkably expressive priors learned from massive data, motivating their integration with forward operators for inverse problems.
A wide range of approaches has emerged, including diffusion-based inverse problem solvers~\citep{kawar2022ddrm},
pseudoinverse or conditional-score constructions~\citep{song2023pseudoinverse},
bridge-based formulations~\citep{wang2025irbridge,debortoli2021dsb,heng2021diffusionbridge},
and posterior-sampling methods such as diffusion posterior sampling~\citep{chung2022dps,li2024ct_dps_jmi,pmlr-v267-alkhouri25b,pmlr-v267-achituve25a,pmlr-v267-ekstrom-kelvinius25b,pmlr-v235-peng24h}.
Despite strong empirical progress, a persistent difficulty is how measurement consistency is enforced.
Most methods rely on step-wise penalties, intermittent projections, or guidance schedules whose strengths and
time dependencies are tuned heuristically~\citep{pmlr-v267-janati25a,pmlr-v235-zheng24f}.
When the inverse problem is severely ill-posed, heuristic enforcement can lead to physically inconsistent yet
visually plausible hallucinations, which is particularly concerning in scientific and medical applications.

At a conceptual level, this difficulty reflects a mismatch in mathematical formulation.
Standard score-based generation is an initial-value stochastic transport:
one starts from a simple reference distribution and simulates a reverse-time SDE to obtain data samples~\citep{song2021scorebased},
built on time-reversal characterizations of diffusions~\citep{anderson1982reverse,pmlr-v235-hirono24a}.
Physics-constrained inference, however, is naturally terminal-value:
the requirement is specified through a feasibility event or an observation.
Injecting penalties or projections into a discrete sampling loop does not, in general, define the path measure
conditioned on a terminal requirement, and can push trajectories away from the intended manifold,
motivating manifold-aware corrections such as MCG~\citep{chung2022mcg}.
These observations point to a missing ingredient: a principled path-space mechanism whose conditioning is posed
directly in terminal form.

In stochastic analysis, terminal-value problems are naturally expressed via backward stochastic differential
equations, BSDEs~\citep{BISMUT1973384,PARDOUX199055}.
A BSDE encodes the terminal requirement explicitly and solves for an adapted pair $(Y_t,Z_t)$.
The process $Z_t$ acts as an adapted control in the martingale representation, steering stochastic dynamics so
that the prescribed terminal specification is satisfied.
This terminal-value viewpoint aligns directly with physics-constrained inference and complements other
path-conditioning paradigms such as diffusion and Schr\"odinger bridges~\citep{debortoli2021dsb,heng2021diffusionbridge}.

We introduce a BSDE-based, terminally conditioned diffusion framework that applies to any pretrained score prior
expressible in Score-SDE form.
We show that, given a pretrained Score-SDE defining a base stochastic dynamics, we can construct an associated BSDE by imposing a terminal condition via a task-defined terminal map.
This map encodes the observation requirement and can be instantiated as a decode, projection onto $\mathcal{S}(y)$,
and encode composition, enforcing physics consistency at the terminal condition rather than through heuristic
time-dependent guidance.
Solving the associated BSDE yields a well-defined inverse mapping from a terminal requirement to a prior state at
a chosen noise level, and it satisfies the terminal specification by construction.
This associated BSDE is not a reverse-time SDE: the pretrained SDE specifies the base diffusion prior, while the
BSDE introduces an adapted pair whose martingale integrand is solved from the terminal requirement.

Our contributions are as follows.
\textbf{First}, we introduce a terminal-conditioned diffusion perspective that brings BSDEs into generative modeling as a
general mechanism for enforcing terminal feasibility.
\textbf{Second}, we provide a rigorous associated-BSDE construction for score-based SDE priors and establish well-posedness
under standard regularity assumptions, ensuring existence and uniqueness of an adapted solution.
\textbf{Third}, we develop a practical solver for general Score-BSDEs that composes arbitrary pretrained diffusion priors
with domain-defined terminal operators through a hard terminal map, avoiding task-specific conditional training and
reducing tuning burden.
\textbf{Finally}, we instantiate the framework on a representative physics-constrained inverse problem, sparse-view CT
reconstruction, as a case study that highlights the methodological gap between terminal-conditioned, feasibility-defined
reconstruction via an associated BSDE and widely used PnP-style step-wise conditioning strategies.

\section{Preliminaries}
\label{sec:prelim}

We review three ingredients used throughout the paper: score-based diffusion in the SDE formulation, physics-constrained inverse problems and common strategies for enforcing measurement consistency summarized in Table~\ref{tab:comparison}, and BSDEs as a terminal-value formalism. This section supports our main objective: a general, well-posed terminal-defined inversion principle applicable to any pretrained score-based SDE prior.

\subsection{Score-Based Diffusion via SDEs}

\noindent\textbf{Scores and denoising score matching:}
Let $p_{\mathrm{data}}$ denote the data distribution on $\mathcal{X}$.
Score-based diffusion models learn the score field $\nabla_x \log p_t(x)$ of a family of perturbed distributions $\{p_t\}_{t\in[0,T]}$.
A standard training principle is denoising score matching.
One samples $x_0 \sim p_{\mathrm{data}}$, selects a noise level $t$, perturbs $x_0$ to obtain $x_t$, and trains a neural network $s_\theta(x_t,t)$ to approximate $\nabla_x \log p_t(x_t)$ \cite{song2021scorebased}.
For Gaussian perturbations of the form $x_t = x_0 + \sigma(t)\epsilon$ with $\epsilon\sim\mathcal{N}(0,I)$, an equivalent objective is
\begin{equation}
\mathcal{L}(\theta)
=
\mathbb{E}_{t,x_0,\epsilon}\Big[
\big\| s_\theta(x_t,t) + \tfrac{x_t-x_0}{\sigma(t)^2}\big\|_2^2
\Big].
\label{eq:dsm}
\end{equation}

\noindent\textbf{Forward-time SDE:}
The SDE viewpoint unifies diffusion learning and sampling in continuous time \cite{song2021scorebased}.
Let $(\Omega,\mathcal{F},\{\mathcal{F}_t\}_{t\in[0,T]},\mathbb{P})$ be a filtered probability space supporting a Brownian motion $(W_t)_{t\in[0,T]}$.
A broad class of perturbation processes can be written as the forward SDE
\begin{equation}
dX_t = f(X_t,t)\,dt + g(t)\,dW_t,\qquad X_0 \sim p_{\mathrm{data}},
\label{eq:forward_sde}
\end{equation}
which induces a marginal density $p_t$ for $X_t$.

\noindent\textbf{Reverse-time SDE:}
Given an estimate $s_\theta(x,t)\approx \nabla_x\log p_t(x)$, samples from $p_{\mathrm{data}}$ can be obtained by simulating the reverse-time SDE \cite{anderson1982reverse,song2021scorebased}
\begin{equation}
dX_t = \big[f(X_t,t) - g(t)^2 s_\theta(X_t,t)\big]\,dt + g(t)\,d\bar W_t,
 t:\ T \rightarrow 0,
\label{eq:reverse_sde}
\end{equation}
initialized from $X_T \sim p_T$.
In practice, Eq.~\eqref{eq:reverse_sde} is integrated numerically, and solver choices influence sample quality \cite{song2021scorebased,karras2022elucidating,jolicoeurmartineau2021gotta,mao2023leapfrog}.
We need to note that the reverse-time SDEs is a sampling mechanism for the learned prior and does not encode terminal feasibility constraints.

\noindent\textbf{Physics-constrained inverse problems:}
We consider inverse problems with measurements defined by Eq.~\eqref{eq:meas}, where $\mathcal{A}$ is a known forward operator.
A central requirement is measurement consistency.
Reconstructed outputs should satisfy the feasible set
\begin{equation}
\mathcal{S}(y) \triangleq \{x:\ \mathcal{A}(x)\approx y\},
\label{eq:feasible_set}
\end{equation}
up to the tolerance implied by the noise model and numerical errors.
This requirement is particularly important in high-stakes settings such as medical imaging, where visually plausible but \textit{\textbf{physically inconsistent hallucinations can be unacceptable}}.

\textbf{Physics constrained Diffusion:}
Table~\ref{tab:comparison} summarizes major paradigms for combining learned priors with physics.
They differ along three axes.
First, some methods enforce constraints through soft penalties, while others apply explicit consistency operators.
Second, some approaches correspond to conditioning a path measure on terminal feasibility, while others apply step-wise corrections without a terminal-value characterization.
Third, practical performance often depends on tuning guidance weights and schedules.
These differences motivate our terminal-conditioned perspective.
Rather than enforcing physics through time-dependent guidance, we will encode feasibility as a terminal condition and solve for an adapted mechanism that achieves it.

\subsection{Backward Stochastic Differential Equations}
A backward stochastic differential equation is specified by a terminal condition and solved backward over time.
Given a terminal random variable $\xi$ and a driver $\hat f$, a standard BSDE seeks adapted processes $(Y_t,Z_t)$ such that \cite{BISMUT1973384,PARDOUX199055}
\begin{equation}
Y_t
=
\xi + \int_t^T \hat f(s,Y_s,Z_s)\,ds - \int_t^T Z_s\,dW_s,
\qquad t\in[0,T].
\label{eq:bsde}
\end{equation}
Equivalently, $dY_t = -\hat f(t,Y_t,Z_t)\,dt + Z_t\,dW_t$ with terminal condition $Y_T=\xi$.
Under standard assumptions that $\hat f$ is Lipschitz in $(Y,Z)$ and $\xi \in L^2$, Eq.~\eqref{eq:bsde} admits a unique adapted solution \cite{PARDOUX199055}.
The process $Z_t$ serves as the martingale integrand that makes the terminal condition achievable within the given filtration.

This terminal-value characterization is distinct from the reverse-time SDE in Eq.~\eqref{eq:reverse_sde}.
Reverse-time SDEs are used to sample from a learned prior by transporting $p_T$ back to $p_{\mathrm{data}}$ using a score field.
In contrast, BSDEs take a terminal specification as an input and solve for an adapted pair $(Y,Z)$ that satisfies it.
This viewpoint is central to our contribution.
\section{Method}
We formalize terminal-conditioned inversion as a general framework for physics-constrained inference under pretrained score-based diffusion priors. Our key observation is that any score-based SDE admits a corresponding BSDE representation, allowing inversion to be posed as a well-defined terminal value problem. By encoding observations or feasibility requirements as terminal conditions, the induced BSDE defines a mathematically well-posed inverse map from terminal constraints to a prior state at a chosen noise level. 

\subsection{BSDEs for Score-based Modeling}
\label{sec:bsde_score}

We introduce the terminal-conditioned inversion problem for arbitrary pretrained score-based SDE priors,
define the associated BSDE representation, and state the well-posedness guarantees that make the inversion map
mathematically sound.
The construction is model-agnostic in the sense that it applies to any prior that admits a score-based SDE form,
and it isolates task information through a terminal specification operator.

\begin{figure*}[t]
    \centering
    \includegraphics[width=\textwidth]{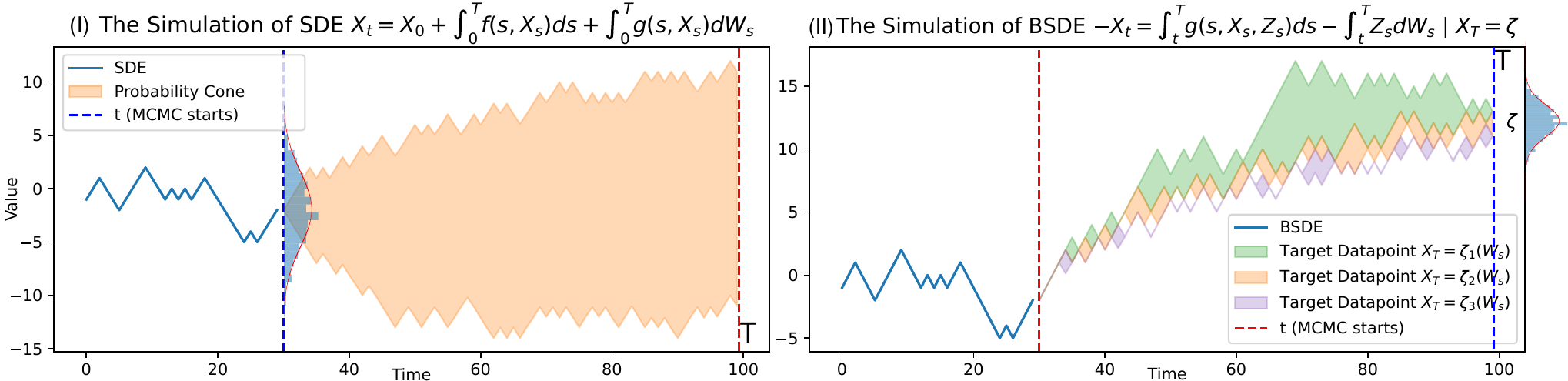}
    \caption{\textbf{Reverse-time SDE versus associated BSDE.}
    Left: reverse-time sampling is an initial value formulation that evolves from a chosen noise law and does not encode
    a fixed terminal specification.
    Right: a BSDE is posed by a terminal condition and solves backward for an adapted representation process that
    enforces the terminal requirement.}
    \label{fig:comparison}
\end{figure*}

\noindent\textbf{Score-SDE priors:}
Let $(\Omega,\mathcal{F},(\mathcal{F}_t)_{t\in[0,T]},\mathbb{P})$ be a filtered probability space supporting a
$d$-dimensional Brownian motion $(W_t)_{t\in[0,T]}$.
A score-based prior is represented through a diffusion process $(X_t)_{t\in[0,T]}$ of the form
\begin{equation}
\mathrm dX_t \;=\; b_\theta(X_t,t)\,\mathrm dt \;+\; \sigma(t)\,\mathrm dW_t,
\label{eq:score_sde_prior}
\end{equation}
where $\sigma(t)$ is a known diffusion scale and the drift $b_\theta$ is specified using a trained score network
$s_\theta(\cdot,t)$.
This class covers standard VP and VE constructions and their score-parametrized variants.


\noindent\textbf{Terminal specification}
Terminal-conditioned inversion is defined by a task-provided terminal requirement.
We encode this requirement through a measurable specification operator
\begin{equation}
\Psi:\mathcal{Y}\to\mathbb{R}^d,
\label{eq:terminal_spec_map}
\end{equation}
which maps an observation $y\in\mathcal{Y}$ to a target terminal latent state.
In practice, $\Psi$ can be implemented by decoding a latent to the image domain, enforcing the task constraint
via a consistency step which its measurement matching under $\mathcal{A}$, and re-encoding back to the
latent space.

\noindent\textbf{SDE Score Associated BSDE}
The key modeling object here is an associated BSDE that pairs with a given SDE defined score prior.
It is posed as a terminal value problem and returns an adapted solution whose initial value is interpreted as the
recovered prior state at a chosen noise level.

\begin{definition}[Associated BSDE for a score prior]
\label{def:associated_bsde}
Fix a horizon $\tau\in(0,T]$.
Let $(X_t)_{t\in[0,\tau]}$ follow the pretrained score prior \eqref{eq:score_sde_prior} driven by the Brownian motion
$(W_t)_{t\in[0,\tau]}$.
Let $\xi$ be an $\mathcal{F}_\tau$-measurable terminal random variable in $L^2(\Omega;\mathbb{R}^d)$.
In terminal-conditioned inversion, we set $\xi=\Psi(y)$ for a given observation $y$.

We seek an adapted pair $(Y,Z)$ with
$Y\in\mathcal{S}^2([0,\tau];\mathbb{R}^d)$ and
$Z\in\mathcal{H}^2([0,\tau];\mathbb{R}^{d\times d})$
that satisfies
\begin{equation}
Y_t
\;=\;
\xi
+\int_t^\tau f_\theta\big(s,X_s,Y_s,Z_s\big)\,\mathrm ds
-\int_t^\tau Z_s\,\mathrm dW_s,
~~ t\in[0,\tau].
\label{eq:associated_bsde}
\end{equation}
We call \eqref{eq:associated_bsde} the associated BSDE induced by the score prior \eqref{eq:score_sde_prior}
and the terminal specification encoded by $\xi$.
\end{definition}

\noindent
Equation \eqref{eq:associated_bsde} is a BSDE in the standard sense.
The terminal condition is imposed at the right endpoint $t=\tau$, and the solution is required to be adapted to
$(\mathcal{F}_t)$.
The process $Z_t$ is an unknown adapted representation process solved from the terminal condition.
This is the structural distinction from reverse-time SDE sampling, where the diffusion coefficient is prescribed and
no terminal constraint is enforced.

\noindent\textbf{Well-posedness:}
To use \eqref{eq:associated_bsde} as a general inversion tool, we need existence and uniqueness of an adapted solution.
In our setting, the BSDE driver is induced by the pretrained score prior through its score network and noise schedule.
Concretely, we consider score-induced drivers of the form
\begin{equation}
f_\theta(t,x,y,z)\;\triangleq\;\hat f\big(t,\,s_\theta(y,t),\,z\big),
\label{eq:score_induced_driver}
\end{equation}
where $\hat f$ is a fixed measurable map determined by the chosen prior family and the terminal-conditioning mechanism.
The analysis follows classical Lipschitz BSDE theory under standard regularity assumptions commonly adopted for neural score models.

\begin{assumption}
\label{ass:bsde_lip}
There exist constants $L_y,L_z\ge 0$ such that for all $t\in[0,\tau]$ and all
$x\in\mathbb{R}^d$, $y,y'\in\mathbb{R}^d$, $z,z'\in\mathbb{R}^{d\times d}$,
\begin{equation}
\big\|f_\theta(t,x,y,z)-f_\theta(t,x,y',z')\big\|
\;\le\;
L_y\,\|y-y'\| \;+\; L_z\,\|z-z'\|.
\label{eq:driver_lip}
\end{equation}
Moreover,
\begin{equation}
\mathbb{E}\int_0^\tau \|f_\theta(t,X_t,0,0)\|^2\,\mathrm dt \;<\;\infty,
\qquad
\mathbb{E}\|\xi\|^2 \;<\;\infty.
\label{eq:driver_integrable}
\end{equation}
\end{assumption}

\begin{theorem}[Well-posedness of the associated BSDE]
\label{thm:wellposed_associated}
Under Assumption~\ref{ass:bsde_lip}, the BSDE \eqref{eq:associated_bsde} admits a unique adapted solution
$(Y,Z)\in\mathcal{S}^2([0,\tau];\mathbb{R}^d)\times\mathcal{H}^2([0,\tau];\mathbb{R}^{d\times d})$.
\end{theorem}

\noindent
Theorem~\ref{thm:wellposed_associated} is proved in the Appendix Sec. \ref{app:proof_wellposed}. The prove follows classical BSDE well-posedness results for Lipschitz drivers
\cite{PARDOUX199055,pengwellposed}.
It provides the mathematical foundation for terminal-conditioned inversion.
Given an observation $y$ and terminal specification $\xi=\Psi(y)$ with $\xi\in L^2$, the inversion map is well-defined by the unique initial value $Y_0$ of the adapted BSDE solution.

The associated BSDE enforces the terminal specification by construction; we have the corollary \ref{cor:terminal_consistency}:
\begin{corollary}[Terminal consistency]
\label{cor:terminal_consistency}
Let $(Y,Z)$ be the unique solution in Theorem~\ref{thm:wellposed_associated}.
Then $Y_\tau=\xi$ holds almost surely.
In terminal-conditioned inversion with $\xi=\Psi(y)$, the terminal feasibility requirement encoded by $\Psi$ is
satisfied almost surely by the BSDE solution.
\end{corollary}

\noindent\textbf{BSDE is fundamentally distinct from reverse-time SDE}
Reverse-time SDE and BSDE are mathematically distinct objects and arise from different problem statements. Reverse-time SDE is a time-reversal characterization of a forward diffusion and is therefore the standard backbone in many plug-and-play (PnP)–style score methods for inverse problems: under suitable regularity assumptions on the forward coefficients and the existence of time-marginal densities, the reversed process is again a diffusion whose
drift is identified by the forward drift together with density-dependent correction terms involving the score of the time marginals \citep{haussmann1986time}.
This time-reversal characterization underlies score-based generative modeling, where the reverse-time dynamics used
for sampling depend on the score of the perturbed data distribution and are integrated from a reference noise law
toward the data end \citep{song2021scorebased}.
In contrast, a BSDE specifies a terminal random variable and seeks an adapted pair $(Y,Z)$ satisfying a backward
integral equation, where $Z$ is the martingale integrand in the Doob--Meyer or martingale representation sense
\citep{PARDOUX199055,el1997backward}.
Accordingly, reverse-time SDE addresses the time-reversal and sampling dynamics of a given diffusion, whereas our
terminal-conditioned inversion is posed as a terminal-value problem whose solution is an adapted process pair determined by the terminal specification.

\begin{figure}[t]
    \centering
    \includegraphics[width=0.5\textwidth]{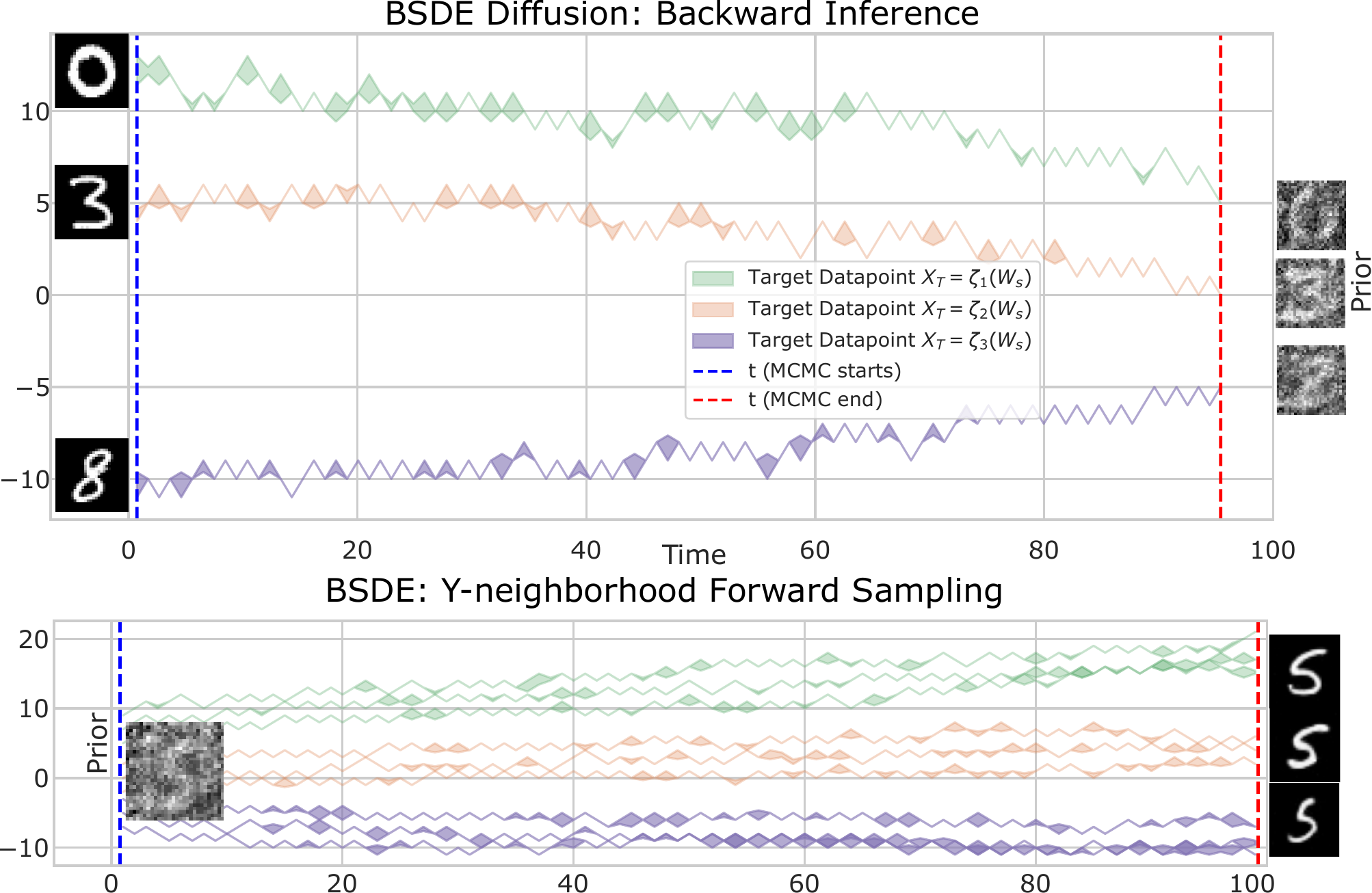}
    \caption{Backward inversion and forward use of the recovered prior state.
    Top: Example of solving the associated BSDE with terminal specification to obtain the recovered prior state $Y_0$.
    Bottom: generate feasible outputs by corresponding Score-SDE.}
    \label{fig:backwardforward}
\end{figure}

\noindent\textbf{Using the recovered prior state via neighborhood sampling:}
\label{forwardgeneration}
Terminal-conditioned inversion returns a recovered prior state $Y_0$ at the chosen noise level.
In inverse problems, a single observation may admit multiple feasible explanations, and exploring this ambiguity is often
desirable for uncertainty quantification.
A practical mechanism is neighborhood sampling around the recovered state.

Let $(Y,Z)$ be the associated BSDE solution for a terminal specification $\xi=\Psi(y)$.
Sample perturbations $\delta$ from a user-chosen local distribution and form perturbed prior states
\begin{equation}
\tilde Y_0 \;=\; Y_0 + \delta.
\label{eq:y0_perturb}
\end{equation}
Each $\tilde Y_0$ is then propagated through the pretrained diffusion dynamics associated with the score prior,
producing a family of outputs whose diversity is controlled by the perturbation scale.
Since all samples share the same terminal specification mechanism through the recovered anchor $Y_0$, this procedure
provides an efficient way to generate multiple plausible solutions without modifying the pretrained prior.

\subsection{Solve the BSDE-based Diffusion Model}
\label{sec:solver}

The associated score-BSDE in Section~\ref{sec:bsde_score} is nonlinear and typically does not admit an analytical
solution.
We therefore rely on numerical solvers to compute the recovered prior state and the adapted representation process.
We present two solver families that are sufficient for our purposes.
The first is a classical regression-based solver that approximates conditional expectations on a time grid.
The second is a deep BSDE solver that parameterizes the representation process with a neural network and optimizes it
by terminal matching, which enables efficient GPU acceleration in our implementation.
A broader overview of BSDE solvers can be found in \cite{Chessari_2023}.

\noindent\textbf{Time discretization:}
Let $0=t_0<t_1<\cdots<t_N=T$ be a uniform grid with step size $\Delta t=T/N$, and denote
$\Delta W_k^i=W_{t_{k+1}}^i-W_{t_k}^i$ for the $i$th Brownian path.
For a BSDE of the form
$Y_t=\xi+\int_t^T f(s,Y_s,Z_s)\,\mathrm ds-\int_t^T Z_s\,\mathrm dW_s$,
a standard explicit discretization yields the recursion
\begin{equation}
y_{k+1}^i
=
y_k^i
-
f(t_k,y_k^i,z_k^i)\,\Delta t
+
z_k^i\,\Delta W_k^i,
~~ k=0,\ldots,N-1,
\label{eq:bsde_discretization}
\end{equation}
where $y_k^i\approx Y_{t_k}$ and $z_k^i\approx Z_{t_k}$ along the $i$th simulated path.

\noindent\textbf{Regression-based solver:}
A classical approach estimates $Z_{t_k}$ by approximating conditional expectations on the grid
\cite{bsde_discrete,Lin_regress,longstaff2001valuing,Chessari_2023}.
Given Monte Carlo (MC) samples $\{(y_k^i,\Delta W_k^i)\}_{i=1}^M$, one uses a basis expansion to regress quantities of the
form $\mathbb{E}\!\left[\Delta W_k\,y_{k+1}\,\middle|\,\mathcal{F}_{t_k}\right]$
onto a chosen finite-dimensional function class.
This yields an estimate of $z_k$ that is then inserted into the forward recursion
\eqref{eq:bsde_discretization}.
The recovered prior state is obtained as the empirical average of the resulting $y_0^i$ values over the MC
paths.

The traditional BSDE solver cannot be used for solving the proposed high dimensional BSDE-Diffusion model.
For high-dimensional state spaces, regression on hand-crafted bases becomes inefficient.
We therefore introduce the learning based solver for solving the proposed BSDE-Diffusion model.
The core idea is to parameterize the representation process by a neural network \cite{deepbsde2} and fit it by minimizing a terminal
matching objective.

We treat the initial value as a learnable parameter $\alpha$ and set $y_0^i=\alpha$ for all paths.
We parameterize $Z_{t_k}$ by a neural network $\Phi_\beta$ and set
$z_k^i=\Phi_\beta(t_k,\eta_k^i)$,
where $\eta_k^i$ denotes the chosen input features at time $t_k$ along the $i$th path.
Given $(\alpha,\beta)$, we propagate $(y_k^i,z_k^i)$ forward using \eqref{eq:bsde_discretization} to obtain $y_N^i$,
and minimize the terminal loss
\begin{equation}
\mathcal{L}(\alpha,\beta)
=
\frac{1}{M}\sum_{i=1}^M \big\|\xi-y_N^i\big\|^2.
\label{eq:deepbsde_loss}
\end{equation}
Gradient-based optimization yields parameters $(\alpha,\beta)$, after which the recovered prior state is given by
$Y_0\approx \alpha$.
Algorithm~\ref{algo:deepbsde} summarizes the procedure. 

\begin{algorithm}
\SetAlgoLined
\DontPrintSemicolon
\SetKwInput{KwInput}{Input}
\SetKwInput{KwOutput}{Output}
\KwInput{Terminal specification $\xi$, step size $\Delta t$, number of paths $M$, initial parameters $(\alpha,\beta)$, learning rate $\lambda_l$}
\KwOutput{Recovered prior state $\alpha$ and network parameters $\beta$}
\KwData{Brownian paths $\{W_{t_k}^i\}_{i=1,k=0}^{M,N}$ and increments $\Delta W_k^i$}
\Repeat{convergence}{
    Set $y_0^i=\alpha$ for $i=1,\ldots,M$\;
    \For{$k=0$ \KwTo $N-1$}{
        $z_k^i=\Phi_\beta(t_k,\eta_k^i)$\;
        $y_{k+1}^i=y_k^i-f(t_k,y_k^i,z_k^i)\Delta t+z_k^i\Delta W_k^i$\;
    }
    Update $(\alpha,\beta)\leftarrow(\alpha,\beta)-\lambda_l\,\nabla_{(\alpha,\beta)}\frac{1}{M}\sum_{i=1}^M\|\xi-y_N^i\|^2$\;
}
\caption{Deep BSDE solver by terminal matching}
\label{algo:deepbsde}
\end{algorithm}

\noindent

\section{Application and Experiment}

Our experiments follow a two-stage narrative that links theory to practice.

\textbf{Part \ref{sec:exp_toy_mnist}) toy validation of BSDE diffusion theory :}
We first use controlled toy benchmarks to empirically validate the theoretical claims of BSDE-diffusion:
(a) terminal-value conditioning enables stable latent inversion which should be consistent with BSDE well-posedness.
(b) the recovered prior encoding is distributionally correct in the sense that neighborhood sampling produces class-consistent outputs with improved distributional agreement.

\textbf{Part \ref{sec:exp_ct}) Physics-constrained CT reconstruction:}
We then evaluate the same terminal-value mechanism on a representative physics-constrained inverse problem: sparse-view CT reconstruction (SVCT) reconstruction.
Here the central objective is terminal physical consistency : reconstructions must satisfy the CT measurement model
within a prescribed numerical tolerance, which is essential for reliable deployment in mission-critical imaging scenarios.

For toy conditional generation, we report distributional similarity via Jensen--Shannon divergence (lower is better)
and shape similarity via average cosine similarity (higher is better).
For CT reconstruction, we report (a) image quality metrics (\textsc{PSNR}/\textsc{SSIM}) and
(b) physics consistency  quantified by the relative measurement residual
$\|\mathcal{A}(x)-y\|/\|y\|$, evaluated {after} applying the terminal consistency  map.

\begin{figure*}
    \centering
    \includegraphics[width=\linewidth]{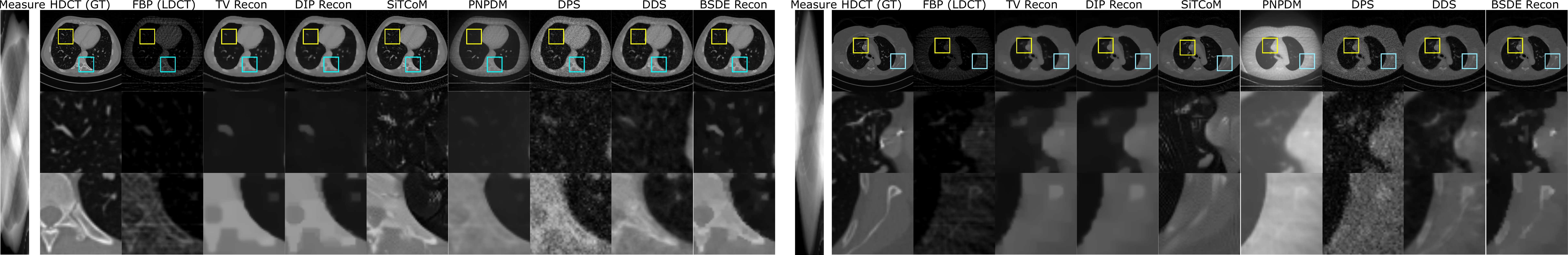}
    \caption{Qualitative comparison for SVCT lung screening.
    From left to right: physics measurement, high resolution reference, sparse-view recon (FBP), TV, DIP-TV, \hl{DPS, DDS*,} SiTCom, PnPDM, and the proposed BSDE.
    Highlighted regions emphasize pulmonary parenchyma and small structural details that are critical for early cancer screening.
    The proposed method better preserves subtle lung textures and low-contrast structures while suppressing noise and \textit{\textbf{remaining consistent with the measurement model}}. More examples are in Appendix. \ref{app:ldctsamples}.}
    \label{fig:ct_qual}
\end{figure*}

\begin{table*}[h]
\centering
\renewcommand{\arraystretch}{0.5}
\setlength{\tabcolsep}{2pt}
\resizebox{1.0\linewidth}{!}{%
\begin{tabular}{lrrrrrrrcccc}
\toprule
Method
& MAE$\downarrow$
& PSNR$\uparrow$
& SSIM$\uparrow$
& NPSNR$\uparrow$
& NSSIM$\uparrow$
& NMSE$\downarrow$
& NCC$\uparrow$
& Task-spec. training?
& Training-free?
& Unconditional prior?
& General prior? \\
\midrule
FBP
& 37.22
& 14.12
& 0.366
& 8.91
& 0.268
& 0.4695
& 0.902
& No
& Yes
& N/A
& N/A \\
TV \cite{tvldct}
& 4.81
& 27.89
& 0.807
& 27.56
& 0.816
& 0.0205
& 0.978
& No
& Yes
& N/A
& N/A \\
TVDIP \cite{diptvldct}
& 4.56
& 28.08
& 0.816
& 27.79
& 0.827
& 0.0197
& 0.979
& No
& Yes
& Yes
& Yes \\
\midrule
PnPDM \cite{zheng2025inversebench}
& 23.38
& 19.06
& 0.551
& 22.36
& 0.691
& 0.3531
& 0.934
& No
& Yes
& Yes
& No \\
\hl{DPS} \cite{chung2022dps}
& 9.84
& 24.27
& 0.494
& 24.36
& 0.508
& 0.0473
& 0.950
& No
& Yes
& Yes
& \hl{Yes} \\
\hl{DDS$^*$} \cite{dds}
& 4.38
& 31.10
& 0.806
& 30.93
& 0.811
& 0.0118
& 0.988
& No
& Yes
& Yes
& \hl{No} \\
SiTCoM \cite{pmlr-v267-alkhouri25b}
& 5.83
& 27.34
& 0.695
& 27.37
& 0.704
& 0.0265
& 0.973
& No
& Yes
& Yes
& \hl{Yes} \\
\midrule
BSDE (Proposed)
& \textbf{3.57}
& \textbf{30.99}
& \textbf{0.868}
& \textbf{30.85}
& \textbf{0.869}
& \textbf{0.0101}
& \textbf{0.989}
& No
& Yes
& Yes
& Yes \\
\bottomrule
\end{tabular}}
\caption{SVCT reconstruction quality. We additionally report method attributes: whether the approach requires task-specific training, is training-free at test time, uses an unconditional prior, and whether the prior is general. *See special footnote [\ref{fn:dds_note}] for DDS.}
\label{tab:ct_results}
\end{table*}

\subsection{Toy Validation}

\noindent\textbf{MNIST: terminal-conditioned inversion and stable prior encodings.}
\label{sec:exp_toy_mnist}
We validate terminal-conditioned inversion in a controlled visual setting using MNIST.
We train a score model under a variance-exploding SDE perturbation and then construct an associated BSDE whose terminal condition is a prescribed target sample $\xi$.
We solve for the adapted pair $(Y_t,Z_t)$ such that $Y_T=\xi$.

We adopt the driver
\[
\hat f\big(s_\theta(Y_t,t),Z_t\big) = -\sigma_t^2 s_\theta(Y_t,t) + \alpha Z_t,
\]
and set $\alpha=0$ to isolate the effect of terminal conditioning.
The corresponding BSDE is
\begin{equation}
\mathrm dY_t = -\sigma_t^2 s_\theta(Y_t,t)\,\mathrm dt + Z_t\,\mathrm dW_t,
\qquad Y_T=\xi.
\label{eq:mnist_bsde}
\end{equation}
This setup probes the terminal-value nature of the model.
The solver learns an adapted process $Z_t$ that enforces the terminal specification.

\noindent\textbf{Stability under repeated solves and neighborhood sampling.}
BSDE inversion maps a terminal sample $\xi$ to an initial state $Y_0$, which serves as a code in the prior space.
Under the well-posedness conditions, the adapted solution is unique, suggesting that repeated solves for the same $\xi$ should produce consistent encodings.
We therefore report the terminal feasibility error $|Y_T-\xi|$ and the variability of the recovered $Y_0$ across solver restarts and random seeds.
To verify that $Y_0$ is a meaningful encoding rather than an optimization artifact, we also sample local neighborhoods
\begin{equation}
\tilde Y_0 = Y_0 + \lambda\epsilon,\qquad \epsilon\sim\mathcal{N}(0,I),
\label{eq:y0_neighborhood}
\end{equation}
and generate outputs by running the BSDE-conditioned forward dynamics from $\tilde Y_0$.
As  Fig.~\ref{fig:application} shows, it yields controlled diversity while preserving digit identity, supports that $Y_0$ is a stable prior code induced by the terminal condition.

\noindent\textbf{Star lightcurves: quantitative distributional consistency}
\label{sec:exp_toy_lightcurves}
MNIST provides an intuitive visualization, but it does not directly quantify distributional fidelity.
We therefore use star lightcurves, where class-conditional generation is commonly evaluated with distributional metrics.
The question is whether terminal-conditioned inversion yields a prior encoding whose neighborhood sampling matches the target conditional distribution.

We compare against representative conditioning mechanisms, including diffusion-based perturbation editing, GAN inversion, and a variational baseline based on an LSTM-VAE.
We follow prior practice in this dataset and report Jensen-Shannon divergence and average cosine similarity for class-conditional generation of Cepheid, Eclipsing Binary, and RR Lyrae lightcurves.
\begin{table}[b]
\centering
\caption{Unsupervised conditional generation on star lightcurves.
Lower is better for Jensen--Shannon divergence.
Higher is better for average cosine similarity.}
\label{tab:starlightcurve}
\resizebox{1.\linewidth}{!}{%
\begin{tabular}{ccccccc}
\hline
\multirow{2}{*}{} & \multicolumn{3}{c}{\textbf{Jensen--Shannon Divergence}$\downarrow$} & \multicolumn{3}{c}{\textbf{Average Cos Similarity}$\uparrow$} \\
\cline{2-7}
 & Cepheid & Eclipsing B & RR Lyrae & Cepheid & Eclipsing B & RR Lyrae \\
\hline
\textbf{SDE Diff.}      & 0.3521 & 0.4397 & 0.3942 & 0.4338 & 0.4648 & 0.6284 \\
\textbf{GAN Inv.}       & 0.3714 & 0.4374 & 0.3639 & 0.2860 & 0.5495 & 0.3302 \\
\textbf{LSTMVAE}        & 0.3929 & 0.4916 & 0.3721 & 0.3794 & 0.0357 & 0.5299 \\
\cdashline{1-7}
\textbf{BSDE Diff.}     & 0.3323 & 0.4223 & 0.3510 & 0.7095 & 0.4321 & 0.6838 \\
\hline
\end{tabular}}
\end{table}
Table~\ref{tab:starlightcurve} shows improved distributional agreement relative to the baselines.
Together with MNIST, this supports the claim that terminal-conditioned inversion yields stable codes and distributionally consistent neighborhood sampling.

\subsection{Real-world Case Study: SVCT Reconstruction}
\label{sec:exp_ct}

We instantiate terminal-conditioned inversion on a physics-constrained inverse problem, SVCT reconstruction.
This case study serves two purposes.
It demonstrates how the terminal condition encodes feasibility under a real measurement operator.
It tests whether the associated BSDE solver can enforce terminal physical consistency without injecting measurements into the pretrained score prior.

\noindent\textbf{Measurement model and hard terminal requirement.}
Let $y_0\in\mathcal Y$ denote the SVCT measurement.
We model the measurement formation by
\[y_0 \;=\; \mathcal A(x_0) + \eta\],
where $x_0\in\mathcal X$ is the unknown standard CT image at the diffusion data end and
$\mathcal A:\mathcal X\to\mathcal Y$ is the known CT forward operator that already includes the SVCT degradation mechanism.
Our goal is terminal physical consistency: the final reconstruction $\hat x$ must satisfy
\[
\|\mathcal A(\hat x)-y_0\| \le \varepsilon \|y_0\|
\]
for a prescribed numerical tolerance $\varepsilon>0$.
We enforce this requirement through a hard terminal specification and a terminal measurement loss used by the solver.
The pretrained score prior remains unconditional and freeze.
\begin{figure}[]
    \centering
    \includegraphics[width=0.5\textwidth]{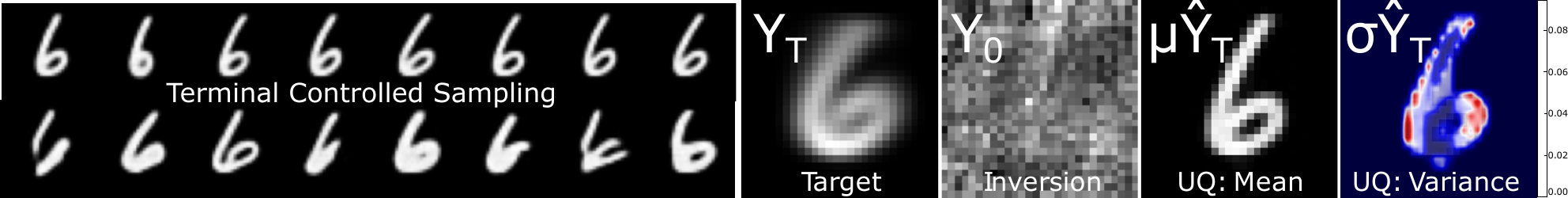}
    \caption{Terminal-conditioned inversion and neighborhood sampling with the BSDE mechanism. See more in Appendix \ref{app:toysamples}}
    \label{fig:application}
\end{figure}

\noindent\textbf{Pretrained SDE Diffusion}
We employ a pretrained score-based diffusion prior \cite{song2021scorebased} in latent space \cite{LDM} and keep it frozen during reconstruction.
Let $X_t\in\mathbb R^d$ denote the latent state and consider the unconditional VE-SDE
\begin{equation}
\mathrm dX_t \;=\; g(t)\,\mathrm dW_t,
\qquad t\in[0,T],
\label{eq:ct_ve_sde}
\end{equation}
where $W_t$ is a standard Brownian motion and $g(t)>0$ is the diffusion scale.
The marginal noise level satisfies $\sigma^2(t)=\int_0^t g^2(s)\,\mathrm ds$.
A pretrained score network $s_\vartheta(x,t)\approx\nabla_x\log p_t(x)$ specifies the denoising field at each time.
The measurement $y_0$ does not enter \eqref{eq:ct_ve_sde}.

\noindent\textbf{Imaging terminal specification:}
We encode the CT measurement requirement as a hard terminal specification through a measurable map
$\Psi:\mathcal Y\to\mathbb R^d$.
Given the observed sinogram $y_0$, $\Psi(y_0)$ returns a data-end latent specification whose decoded image is
measurement-consistent under $\mathcal A$ within the prescribed tolerance.
This exposes the CT constraint to the BSDE through a terminal boundary condition rather than through the pretrained score drift.

\noindent\textbf{SDE-associated BSDE Form:}
Fix a diffusion noise level $\tau$ and solve an associated BSDE on the interval $[0,\tau]$.
We use the BSDE time variable $t\in[0,\tau]$ as an inversion-time parameter, mapped to diffusion time $r=\tau-t$,
so that $t=\tau$ corresponds to the diffusion data end $r=0$ and $t=0$ corresponds to the diffusion state at noise level $r=\tau$.
The BSDE terminal condition is
\begin{equation}
Y_\tau \;=\; \Psi(y_0),
\label{eq:ct_bsde_terminal_bc}
\end{equation}
and the associated BSDE reads
\begin{equation}
Y_t
\;=\;
Y_\tau
+\int_t^\tau f\big(s,X_s,Y_s,Z_s\big)\,\mathrm ds
-\int_t^\tau Z_s\,\mathrm dW_s,
~~ t\in[0,\tau].
\label{eq:ct_bsde_cont_correct}
\end{equation}
Equivalently, $\mathrm dY_t=-f(t,X_t,Y_t,Z_t)\,\mathrm dt+Z_t\,\mathrm dW_t$ with terminal condition
\eqref{eq:ct_bsde_terminal_bc}.
The inversion target is the recovered noisy prior state at diffusion noise level $\tau$, represented by the BSDE initial value $Y_0$.

We discretize the VE-SDE using a schedule $\{\sigma_k\}_{k=0}^K$ and increments
$\Delta\sigma_k^2=\sigma_{k+1}^2-\sigma_k^2$.
We use the pretrained score-defined denoising recursion from index $k_\tau$ down to $0$ to map the inferred latent state at noise level $\tau$
to the data end.
Let $\widetilde X_0$ denote the resulting data-end latent and let $D$ denote the fixed decoder mapping latent states.
We obtain $\widetilde x_0=D(\widetilde X_0)$ and set the final reconstruction as:
$\hat x \;\triangleq\; \widetilde x_0$.

\noindent\textbf{Terminal measurement consistency:}
We solve the BSDE by minimizing the measurement discrepancy
\begin{equation}
\min \ \mathcal L_{\mathrm{meas}}
\;\triangleq\;
\mathbb E\Big[\big\|\mathcal A(\hat x)-y_0\big\|^2\Big],
\label{eq:ct_meas_loss}
\end{equation}
where the expectation is taken over the innovation variables used by the discrete score-defined recursion.
This loss enforces terminal physical consistency in the measurement domain while keeping the pretrained prior unconditional:
we do not inject $y_0$ into the SDE coefficients nor into score evaluation, and instead solve for an adapted BSDE solution whose inferred
prior state yields a terminal reconstruction consistent with $\mathcal A$ and $y_0$.

\noindent\textbf{Baseline and Result}
We compare against filtered back-projection (FBP), total variation minimization (TV) \cite{tvldct}, deep image prior (TVDIP) \cite{diptvldct}, and four diffusion-based inverse problem solvers: PnPDM \cite{zheng2025inversebench}, DPS \cite{chung2022dps}, DDS \cite{dds}, and SiTCoM \cite{pmlr-v267-alkhouri25b}.
These baselines are chosen to cover complementary reconstruction paradigms: analytic and variational methods (FBP, TV), a training-free DIP baseline (TVDIP), and representative diffusion-based reconstruction methods.
All methods use the same $\mathcal{A}$ and are evaluated under the same metrics.
FBP serves as a fast analytic baseline, TV \cite{tvldct} as a classical physics-constrained regularization baseline, TVDIP \cite{diptvldct} as a training-free DIP baseline, and PnPDM \cite{zheng2025inversebench}, DPS \cite{chung2022dps}, DDS\footnote{DDS \cite{dds} tested here uses CT pretrained prior, it is included here as additional baselines rather than main fair-comparison.\label{fn:dds_note}} \cite{dds}, and SiTCoM \cite{pmlr-v267-alkhouri25b} as diffusion-based baselines.

Quantitatively, Table~\ref{tab:ct_results} summarizes SVCT reconstruction performance on the SVCT dataset.
The proposed BSDE solver achieves the best MAE ($3.57$), SSIM ($0.868$), NSSIM ($0.869$), NMSE ($0.0101$), and NCC ($0.989$), while remaining competitive in PSNR/NPSNR.
Relative to the non-diffusion baselines, MAE decreases from $37.22$ under FBP and $4.56$ under TVDIP to $3.57$, while SSIM increases from $0.366$ under FBP and $0.816$ under TVDIP to $0.868$.
Compared with the diffusion baselines built without dataset-specific retraining in our setting, BSDE substantially improves over PnPDM, DPS, and SiTCoM.
We note that DDS uses a released \emph{CT-domain prior}, so it does not correspond to exactly the same general-prior setting studied in this paper.
We report DDS for reference in Table~\ref{tab:ct_results}, but do not make it the main focus of the discussion here. A more detailed discussion of this mismatch and the resulting fairness issue is deferred to the Appendix.

Figure~\ref{fig:ct_qual} corroborates the quantitative trends.
FBP and TVDIP retain residual streaking and structured artifacts.
TV suppresses noise but over-smooths fine structures in the highlighted regions.
Among the diffusion baselines, PnPDM, DPS, and SiTCoM can sharpen details but may also exhibit contrast drift or reduced structural fidelity under step-wise soft measurement guidance.
In contrast, BSDE better preserves subtle textures and edges while maintaining measurement consistency through the terminal feasibility map, yielding reconstructions that remain visually faithful and physically constrained.

\section{Ablation}

\paragraph{Effect of terminal conditioning.}
We first conduct a controlled ablation at resolution $128 \times 128$ to isolate the contribution of terminal conditioning within the proposed BSDE framework. We compare three variants. 
\textbf{Hard terminal constraint} denotes the full model, where the BSDE correction network is optimized with the terminal measurement loss. 
\textbf{Soft terminal penalty} replaces the hard terminal conditioning mechanism with a soft penalty toward the initial terminal estimate $z_{\tau,\mathrm{init}}$. 
\textbf{No terminal constraint} removes the BSDE correction network and directly optimizes the terminal latent variable $z_{\tau}$.

\begin{table}[t]
\centering
\caption{Ablation of terminal conditioning on $128\times128$ CT reconstruction. The hard terminal constraint gives the best reconstruction quality and the lowest measurement residual, confirming the importance of terminal feasibility in the BSDE formulation.}
\label{tab:ablation_terminal_conditioning}
\resizebox{\linewidth}{!}{
\begin{tabular}{lcccc}
\toprule
Method & PSNR$\uparrow$ & SSIM$\uparrow$ & NMSE$\downarrow$ & Residual$\downarrow$ \\
\midrule
Hard terminal constraint 
& $\mathbf{35.36 \pm 2.06}$ 
& $\mathbf{0.9460 \pm 0.0165}$ 
& $\mathbf{0.00352 \pm 0.00224}$ 
& $\mathbf{0.00128 \pm 0.00021}$ \\
Soft terminal penalty 
& $33.10 \pm 0.78$ 
& $0.9194 \pm 0.0072$ 
& $0.00543 \pm 0.00174$ 
& $0.00203 \pm 0.00022$ \\
No terminal constraint 
& $17.38 \pm 0.75$ 
& $0.4188 \pm 0.0695$ 
& $0.2033 \pm 0.0689$ 
& $0.2008 \pm 0.0238$ \\
\bottomrule
\end{tabular}}
\end{table}

Table~\ref{tab:ablation_terminal_conditioning} shows that terminal conditioning is essential for both image quality and measurement consistency. Replacing the hard terminal constraint with a soft penalty already degrades PSNR from $35.36$ to $33.10$ and increases the residual from $0.00128$ to $0.00203$, indicating that a soft attraction to $z_{\tau,\mathrm{init}}$ does not fully preserve terminal feasibility. Removing the terminal constraint causes a much larger failure: PSNR drops to $17.38$, SSIM decreases to $0.4188$, and the residual increases by two orders of magnitude to $0.2008$. This confirms that the performance gain is not merely due to latent optimization, but comes from the BSDE correction mechanism that enforces the terminal measurement requirement.

\paragraph{Effect of terminal timepoint $\tau$.}
We further perform a standalone sweep over the BSDE terminal timepoint 
$\tau \in \{0.05, 0.15, 0.35, 0.45, 0.50\}$ to isolate the effect of the inversion horizon. The terminal timepoint controls the noise level at which the recovered prior state is anchored. A smaller $\tau$ keeps the inversion closer to the data end and can improve reconstruction fidelity, while a larger $\tau$ requires a stronger correction from a noisier latent state and can make terminal recovery less stable.

Empirically, $\tau=0.05$ achieves the highest mean PSNR, $35.52 \pm 2.20$, but it also requires a stronger terminal correction. In contrast, $\tau=0.15$ achieves a nearly identical PSNR of $35.21 \pm 1.84$ with better stability, providing the best quality--stability tradeoff in our experiments. When $\tau$ becomes larger, performance drops substantially: PSNR decreases to $31.34 \pm 1.69$ at $\tau=0.35$ and further to $23.48 \pm 3.29$ at $\tau=0.45$. These results suggest that terminal-conditioned BSDE inversion benefits from anchoring at a moderate diffusion noise level.

\section{Conclusion}

We proposed terminal-conditioned inversion for score-based SDE priors via an associated BSDE.
Feasibility is imposed through a terminal specification, while the pretrained score prior remains unconditional and frozen.
Under standard regularity conditions, the resulting BSDE is well posed and defines a unique adapted inverse mapping from the terminal requirement to a prior anchor.
We also developed a neural BSDE solver that makes this construction numerically tractable for complex score priors without conditional retraining.
Toy experiments validate stable inversion and distributionally consistent sampling.
As a case study, SVCT reconstruction shows improved fidelity together with measurement consistency.
Future work will focus on faster solvers and broader application cases.

\section*{Acknowledgment}
This research is supported by the Ruth S. Holmberg Grant, and is supported
in part by the Provost’s Office of University of Tennessee at Chattanooga. We thank the anonymous reviewers for their constructive feedback.

\section*{Impact Statement}

This paper presents work whose goal is to advance the field of Machine Learning. There are many potential societal consequences of our work, none which we feel must be specifically highlighted here.


\bibliography{example_paper}
\bibliographystyle{icml2026}

\newpage
\appendix
\onecolumn
\section{Notations}
We collect the main notation used throughout the paper.

\begin{itemize}
\item \textbf{Probability space.}
$(\Omega,\mathcal F,(\mathcal F_t)_{t\in[0,T]},\mathbb P)$ is a filtered probability space.
$(W_t)_{t\in[0,T]}$ is a $d$-dimensional Brownian motion adapted to $(\mathcal F_t)$.

\item \textbf{Data and distributions.}
$\mathcal X\subseteq\mathbb R^d$ is the data space.
$p_{\mathrm{data}}$ is the data distribution on $\mathcal X$.
For a stochastic process $(X_t)$, $p_t$ denotes the marginal density of $X_t$ at time $t$.

\item \textbf{Score field and network.}
The score field at time $t$ is $\nabla_x \log p_t(x)$.
$s_\theta(x,t)$ is a neural network approximation of $\nabla_x \log p_t(x)$, with parameters $\theta$.

\item \textbf{Score-SDE prior.}
$T$ is the diffusion time horizon.
The forward Score-SDE prior is
\[
\mathrm dX_t=f(X_t,t)\,\mathrm dt+g(t)\,\mathrm dW_t,\qquad X_0\sim p_{\mathrm{data}}.
\]
When needed, $\bar W_t$ denotes a Brownian motion for reverse-time sampling dynamics.

\item \textbf{Inverse problems and measurements.}
$A$ is a known forward operator and $y$ is the measurement.
$S(y)$ denotes the measurement-feasible set (or an $\varepsilon$-feasible set) specified by the task.

\item \textbf{Terminal specification.}
$\Psi$ is a measurable terminal specification operator that encodes the feasibility requirement induced by $y$.
The BSDE terminal random variable is $\xi\coloneqq\Psi(y)$.

\item \textbf{Associated BSDE (terminal-conditioned inversion).}
$\tau\in(0,T]$ is the chosen inversion horizon (noise level).
We solve an adapted pair $(Y_t,Z_t)_{t\in[0,\tau]}$ satisfying
\[
Y_t=\xi+\int_t^\tau f_\theta(s,X_s,Y_s,Z_s)\,\mathrm ds-\int_t^\tau Z_s\,\mathrm dW_s,
\qquad t\in[0,\tau],
\]
where $f_\theta$ is the BSDE driver.
The recovered prior state is $Y_0$.

\item \textbf{Function spaces and regularity.}
$S^2([0,\tau];\mathbb R^d)$ denotes square-integrable adapted processes with finite $\mathbb E[\sup_{t\in[0,\tau]}\|Y_t\|^2]$.
$H^2([0,\tau];\mathbb R^{d\times d})$ denotes square-integrable predictable processes with finite
$\mathbb E[\int_0^\tau \|Z_t\|^2\,\mathrm dt]$.
$\xi\in L^2(\Omega;\mathbb R^d)$ means $\mathbb E\|\xi\|^2<\infty$.
$L_y,L_z$ are Lipschitz constants of $f_\theta$ in $(y,z)$.

\item \textbf{Discrete noise schedule.}
$\{\sigma_k\}_{k=0}^K$ is a discrete noise schedule and $\Delta\sigma_k^2\coloneqq\sigma_{k+1}^2-\sigma_k^2$.
$\epsilon_k\sim\mathcal N(0,I)$ are standard Gaussian innovations.
\end{itemize}

\section{Proofs}
\label{sec:proofs}

\subsection{Notation and function spaces}

Let $(\Omega,\mathcal{F},(\mathcal{F}_t)_{t\in[0,T]},\mathbb{P})$ support a $d$-dimensional Brownian motion
$(W_t)_{t\in[0,T]}$.
We use the standard spaces (as in Pardoux--Peng's proof strategy):
\[
\mathcal{H}^2 := \Big\{ Z:\Omega\times[0,T]\to\mathbb{R}^{d\times d}\ \text{progressively measurable}:
\ \mathbb{E}\int_0^T \|Z_t\|^2 dt < \infty \Big\},
\]
\[
\mathcal{S}^2 := \Big\{ Y:\Omega\times[0,T]\to\mathbb{R}^{d}\ \text{adapted, continuous}:
\ \mathbb{E}\big[\sup_{t\in[0,T]}\|Y_t\|^2\big] < \infty \Big\}.
\]
(We use the Euclidean norm $\|\cdot\|$ and the associated inner product $\langle\cdot,\cdot\rangle$.)

We consider the BSDE
\begin{equation}
\label{eq:app_bsde}
Y_t=\xi+\int_t^T \hat f\!\big(s_\theta(Y_s,s),Z_s,s\big)\,ds-\int_t^T Z_s\,dW_s,\qquad t\in[0,T].
\end{equation}
For convenience, define the effective driver
\[
G(t,y,z):=\hat f\!\big(t, s_\theta(y,t), z\big),
\]
where $s_\theta(\cdot,t):\mathbb{R}^d\to\mathbb{R}^m$ and
$\hat f(t,\cdot,\cdot):\mathbb{R}^m\times\mathbb{R}^{d\times d}\to\mathbb{R}^d$.
Then \eqref{eq:app_bsde} can be rewritten as
\[
Y_t=\xi+\int_t^T G(s,Y_s,Z_s)\,ds-\int_t^T Z_s\,dW_s.
\]

\subsection{Step A: Properties of the driver $G$ (from Assumption~\ref{ass:bsde_lip})}
\label{app:proof_wellposed}

\begin{lemma}[Lipschitzness of $G$ in $(y,z)$]
\label{lem:app_driver_lip}
Under Assumption~\ref{ass:bsde_lip}, for all $t\in[0,T]$ and all $(y,z),(y',z')$,
\[
\|G(t,y,z)-G(t,y',z')\|
\le L_y\|y-y'\|+L_z\|z-z'\|.
\]
\end{lemma}
This is exactly Assumption~\ref{ass:bsde_lip}.

\begin{lemma}[Square-integrability at the origin]
\label{lem:app_origin_H2}
Under Assumption~\ref{ass:bsde_lip},
\[
\mathbb{E}\int_0^T \|G(t,0,0)\|^2 dt < \infty.
\]
\end{lemma}

This is the integrability condition in Assumption~\ref{ass:bsde_lip}.

Hence, under the assumptions of Theorem~\ref{thm:wellposed_associated} and Proposition~\ref{ass:bsde_lip},
$G$ is uniformly Lipschitz in $(y,z)$ (Lemma~\ref{lem:app_driver_lip}) and satisfies the standard
square-integrability condition at the origin (Lemma~\ref{lem:app_origin_H2}).

\subsection{Step B: Linear BSDE solvability via martingale representation}

\begin{lemma}[Linear BSDE (construction)]
\label{lem:app_linear}
Let $\xi\in L^2(\Omega;\mathbb{R}^d)$ and let $g\in L^2(\Omega\times[0,T];\mathbb{R}^d)$ be progressively measurable.
Then there exists a unique pair $(Y,Z)\in\mathcal S^2\times\mathcal H^2$ such that
\begin{equation}
\label{eq:app_linear}
Y_t=\xi+\int_t^T g(s)\,ds-\int_t^T Z_s\,dW_s,\qquad t\in[0,T].
\end{equation}
\end{lemma}

\begin{proof}
Define $U:=\xi+\int_0^T g(s)\,ds\in L^2(\Omega;\mathbb{R}^d)$.
Let $M_t:=\mathbb{E}[U\mid\mathcal F_t]$, then $(M_t)$ is a square-integrable martingale.
By the martingale representation theorem, there exists a unique $Z\in\mathcal H^2$ such that
$M_t=M_0+\int_0^t Z_s\,dW_s$.
Set $Y_t:=M_t-\int_0^t g(s)\,ds$. Rearranging yields \eqref{eq:app_linear}.
Uniqueness follows by subtracting two solutions: the difference has $\xi=0$ and $g\equiv 0$,
so its martingale representation forces $Z\equiv 0$ and hence $Y\equiv 0$.
\end{proof}

\subsection{Step C: Uniqueness for the nonlinear BSDE (It\^o energy estimate)}

\begin{proposition}[Uniqueness in $\mathcal S^2\times\mathcal H^2$]
\label{prop:app_unique}
Assume $\xi\in L^2$ and $G$ satisfies the uniform Lipschitz condition in Lemma~\ref{lem:app_driver_lip}.
Then \eqref{eq:app_bsde} has at most one solution in $\mathcal S^2\times\mathcal H^2$.
\end{proposition}

\begin{proof}
Let $(Y,Z)$ and $(\tilde Y,\tilde Z)$ be two solutions. Set
$\Delta Y:=Y-\tilde Y$, $\Delta Z:=Z-\tilde Z$, and
$\Delta G(t):=G(t,Y_t,Z_t)-G(t,\tilde Y_t,\tilde Z_t)$.
Then
\[
\Delta Y_t=\int_t^T \Delta G(s)\,ds-\int_t^T \Delta Z_s\,dW_s.
\]
Apply It\^o's formula to $\|\Delta Y_t\|^2$ on $[t,T]$:
\[
\|\Delta Y_t\|^2+\int_t^T\|\Delta Z_s\|^2ds
=2\int_t^T\langle \Delta Y_s,\Delta G(s)\rangle ds
-2\int_t^T\langle \Delta Y_s,\Delta Z_s\,dW_s\rangle.
\]
Take expectation; the stochastic integral has zero expectation:
\[
\mathbb{E}\|\Delta Y_t\|^2+\mathbb{E}\int_t^T\|\Delta Z_s\|^2ds
=2\mathbb{E}\int_t^T\langle \Delta Y_s,\Delta G(s)\rangle ds.
\]
Using Cauchy--Schwarz and Young's inequality with $\varepsilon=1$,
\[
2\langle a,b\rangle \le \|a\|^2+\|b\|^2.
\]
By Lipschitzness (Lemma~\ref{lem:app_driver_lip}),
$\|\Delta G(s)\|\le L_G(\|\Delta Y_s\|+\|\Delta Z_s\|)$, hence
$\|\Delta G(s)\|^2\le 2L_G^2(\|\Delta Y_s\|^2+\|\Delta Z_s\|^2)$.
Therefore,
\[
\mathbb{E}\|\Delta Y_t\|^2
\le (1+2L_G^2)\int_t^T \mathbb{E}\|\Delta Y_s\|^2 ds,
\]
where we dropped the nonnegative $\mathbb{E}\int_t^T\|\Delta Z_s\|^2ds$ term on the left.
By Gr\"onwall's inequality, $\mathbb{E}\|\Delta Y_t\|^2=0$ for all $t$, hence $\Delta Y\equiv 0$ a.s.
Plugging back gives $\mathbb{E}\int_0^T\|\Delta Z_s\|^2ds=0$, so $\Delta Z\equiv 0$.
\end{proof}

\subsection{Step D: Existence via Picard iteration (Pardoux--Peng proof route \cite{PARDOUX199055,pengwellposed})}

\begin{theorem}[Existence and uniqueness (detailed proof of Theorem~\ref{thm:wellposed_associated})]
\label{thm:app_wellposed}
Assume:
(i) $\xi\in L^2(\Omega;\mathbb{R}^d)$;
(ii) $\mathbb{E}\int_0^T \|G(t,0,0)\|^2dt<\infty$ (e.g., Lemma~\ref{lem:app_origin_H2});
(iii) $G$ is uniformly Lipschitz in $(y,z)$ (Lemma~\ref{lem:app_driver_lip}).
Then \eqref{eq:app_bsde} admits a unique adapted solution $(Y,Z)\in\mathcal S^2\times\mathcal H^2$.
\end{theorem}

\begin{proof}
{(1) Picard scheme (linearization).}
Initialize $(Y^0,Z^0)\equiv(0,0)$.
Given $(Y^n,Z^n)$, define the progressively measurable process
\[
g^{n+1}(t):=G(t,Y^n_t,Z^n_t).
\]
We claim $g^{n+1}\in L^2(\Omega\times[0,T])$.
Indeed, by Lipschitzness,
\[
\|G(t,y,z)\|\le \|G(t,0,0)\|+L_G(\|y\|+\|z\|),
\]
hence with $(a+b+c)^2\le 3(a^2+b^2+c^2)$,
\[
\|g^{n+1}(t)\|^2
\le 3\|G(t,0,0)\|^2+3L_G^2\|Y^n_t\|^2+3L_G^2\|Z^n_t\|^2.
\]
Integrating and taking expectation gives finiteness because
$\mathbb{E}\int_0^T\|G(t,0,0)\|^2dt<\infty$ and (inductively) $(Y^n,Z^n)\in\mathcal S^2\times\mathcal H^2$
implies $\mathbb{E}\int_0^T\|Y^n_t\|^2dt\le T\,\mathbb{E}\sup_{t}\|Y^n_t\|^2<\infty$ and
$\mathbb{E}\int_0^T\|Z^n_t\|^2dt<\infty$.

Therefore, by Lemma~\ref{lem:app_linear}, there exists a unique
$(Y^{n+1},Z^{n+1})\in\mathcal S^2\times\mathcal H^2$ solving the linear BSDE
\begin{equation}
\label{eq:app_picard}
Y^{n+1}_t=\xi+\int_t^T g^{n+1}(s)\,ds-\int_t^T Z^{n+1}_s\,dW_s.
\end{equation}

{(2) Contraction estimate in an exponentially weighted norm.}

Let $\Delta Y^{n+1}:=Y^{n+1}-Y^n$, $\Delta Z^{n+1}:=Z^{n+1}-Z^n$.
Subtract \eqref{eq:app_picard} at levels $n+1$ and $n$ to get
\[
\Delta Y^{n+1}_t=\int_t^T \Delta g^{n+1}(s)\,ds-\int_t^T \Delta Z^{n+1}_s\,dW_s,
\quad
\Delta g^{n+1}(t):=g^{n+1}(t)-g^{n}(t).
\]
Fix $\beta>0$ and apply It\^o's formula to $e^{\beta t}\|\Delta Y^{n+1}_t\|^2$:
\begin{align*}
&e^{\beta t}\mathbb E\|\Delta Y^{n+1}_t\|^2
+\mathbb E\int_t^T e^{\beta s}\Big(\beta\|\Delta Y^{n+1}_s\|^2+\|\Delta Z^{n+1}_s\|^2\Big)\,ds \\
&\qquad=2\mathbb E\int_t^T e^{\beta s}\langle \Delta Y^{n+1}_s,\Delta g^{n+1}(s)\rangle ds,
\end{align*}
where the stochastic integral vanishes after expectation.
Using Young's inequality,
\[
2\langle a,b\rangle\le \frac{\beta}{2}\|a\|^2+\frac{2}{\beta}\|b\|^2,
\]
we obtain
\[
\mathbb E\int_t^T e^{\beta s}\Big(\tfrac{\beta}{2}\|\Delta Y^{n+1}_s\|^2+\|\Delta Z^{n+1}_s\|^2\Big)\,ds
\le \frac{2}{\beta}\mathbb E\int_t^T e^{\beta s}\|\Delta g^{n+1}(s)\|^2ds.
\]
By Lipschitzness of $G$,
\[
\|\Delta g^{n+1}(s)\|
=\|G(s,Y^n_s,Z^n_s)-G(s,Y^{n-1}_s,Z^{n-1}_s)\|
\le L_G(\|\Delta Y^{n}_s\|+\|\Delta Z^{n}_s\|),
\]
hence $\|\Delta g^{n+1}(s)\|^2\le 2L_G^2(\|\Delta Y^{n}_s\|^2+\|\Delta Z^{n}_s\|^2)$.
Define the weighted energy norm
\[
\|(U,V)\|_{\beta,t}^2:=\mathbb E\int_t^T e^{\beta s}\big(\|U_s\|^2+\|V_s\|^2\big)\,ds.
\]
The above inequalities yield
\[
\|(\Delta Y^{n+1},\Delta Z^{n+1})\|_{\beta,t}^2
\le \frac{4L_G^2}{\beta}\|(\Delta Y^{n},\Delta Z^{n})\|_{\beta,t}^2.
\]
Choose $\beta>4L_G^2$ and set $q:=\frac{4L_G^2}{\beta}\in(0,1)$.
Then $\|(\Delta Y^{n+1},\Delta Z^{n+1})\|_{\beta,0}^2 \le q\,\|(\Delta Y^{n},\Delta Z^{n})\|_{\beta,0}^2$,
so $(Y^n,Z^n)$ is Cauchy under $\|\cdot\|_{\beta,0}$ and thus converges to some $(Y,Z)$ in that norm.

\noindent\textbf{(3) Passing to the limit and verification.}
Since $G$ is Lipschitz and $(Y^n,Z^n)\to(Y,Z)$ in the weighted $L^2$ sense,
it follows that $G(\cdot,Y^n,Z^n)\to G(\cdot,Y,Z)$ in $L^2(\Omega\times[0,T])$.
Letting $n\to\infty$ in \eqref{eq:app_picard} yields that $(Y,Z)$ satisfies \eqref{eq:app_bsde}.

\noindent\textbf{(4) Membership in $\mathcal S^2\times\mathcal H^2$.}
From \eqref{eq:app_bsde} and the square-integrability of $\xi$ and $G(\cdot,0,0)$, standard It\^o/BDG
a priori estimates imply $\mathbb E\sup_{t\in[0,T]}\|Y_t\|^2+\mathbb E\int_0^T\|Z_t\|^2dt<\infty$,
hence $(Y,Z)\in\mathcal S^2\times\mathcal H^2$.

Thus, uniqueness follows from Proposition~\ref{prop:app_unique}.
\end{proof}

\noindent

Theorem~\ref{thm:wellposed_associated} in the main text is exactly Theorem~\ref{thm:app_wellposed} above,
applied to $G(t,y,z)=\hat f(t,s_\theta(y,t),z)$ with the driver/score conditions in Proposition~\ref{ass:bsde_lip}.

\section{Key Implementations}
\subsection{Backward Stochastic Differential Equation for Diffusion}

We consider the case used in our example that the Forward Diffusion model defined through:
\begin{equation}
\underbrace{Y_t}_{\text{Returned Value $Y_t$ in BSDE()}}  = 
\underbrace{\xi}_{\text{Target Xi()}}
+ \int_t^T  \underbrace{(\sigma^2u)S_\theta(Y_u, u) }_{\text{Driver f()}} ds
- \int_t^T \underbrace{Z_s}_{\text{Solver network in BSDE(). }} \underbrace{d\brwonian_s}_{\text{The dw in BSDE()}}
\end{equation}

where $\hat{f}=(\sigma^2u)S_\theta(Y_u, u)$. The Python implementation of above BSDE-based Diffusion is:
\begin{pythoncode}
# Driver Func
def f(t, y, z, score_model, config):

    t = torch.ones(config.solver_batch_size, device=device) * t

    score_model.eval()
    model_output = score_model(y, t)
    #time schedule
    alphas = 1.0 - betas 
    alphas_cumprod = alphas.cumprod(dim=0)
    weighted_score = betas / torch.sqrt(1 - alphas_cumprod)
    
    #Extract coefficients based on t
    mean = extract(1 / torch.sqrt(alphas), t, y.shape) * 
    (- extract(weighted_score, t, y.shape) * model_output)

    return -mean
    
# Target Func
def Xi(target):
    # Target Y_T
    return torch.Tensor(target).to(device)
    
# Integral Func
def BSDE(batch_size, N):
    
    delta_t = T / N
    W = torch.randn(batch_size, dim_d, N, device=device) 
        * np.sqrt(delta_t) # Brownian
    t = torch.ones(batch_size, device=device)

    y = y_0 * marginal_prob_std(t)[:, None, None, None] +
        torch.zeros(W.size()[0],dim_y,device=device)
        
    # Integral over t
    for i in range(N):
        z = PHI(i, delta_t) # Solver Net
        dw = torch.randn_like(z)  * np.sqrt(delta_t)

        t = torch.ones(batch_size, device=device) * i * delta_t
        g = (z * diffusion_coeff_fn(t)[:, None, None, None])
        rdm = (dw * g).reshape(-1, dim_y)
        y_t = y = y - f(delta_t*i, x, y, z)*delta_t + rdm
    return y_t
\end{pythoncode}

In the aforementioned code fragment, the python function \texttt{f(t, y, z, score\_model, config)} aligns with the driver function explicated in Eq. \ref{eq:bsde}, while the function \texttt{Xi(target)} represents the target $Y_T = \xi$. The function \texttt{BSDE(batch\_size, N)} integrates the BSDE from $t$ to $T$ across $N$ steps of discretization.

We can resolve the presented BSDE by employing a numerical solver as delineated in the manuscript.

\section{Algorithms}
\label{app:forward_algorithms}
Once the backward BSDE is solved, we obtain an initial state $Y_0$ and an adapted
control process $(Z_t)_{t\in[0,T]}$ (or an estimator thereof). We then construct controlled {forward}
samplers by perturbing $Y_0$ (\textit{$Y_0$-neighborhood sampling}).

\noindent\textbf{Forward-time form.}
From \eqref{eq:bsde}, the corresponding forward-time stochastic dynamics read
\begin{equation}
\label{eq:forward_sde_from_bsde}
dY_t = -\,\hat f\!\big(s_\theta(Y_t,t),\,Z_t,\,t\big)\,dt \;+\; Z_t\,dW_t,\qquad Y_0\ \text{given}.
\end{equation}
All forward samplers below are Euler--Maruyama discretizations of \eqref{eq:forward_sde_from_bsde}.

In practice, we use a learned estimator $\Phi$ (trained by our BSDE solver) to approximate the control,
e.g., $Z_{t_i}\approx \Phi(t_i, W_{t_i};\beta)\in\mathbb R^{d\times d}$. 
\subsection*{Algorithm~\ref{algo:forwardsampling_fix}: $Y_0$-neighborhood sampling}

\noindent
Given a reference initial state $Y_0$ inferred by solving the backward BSDE for a terminal datum
$\xi$, we generate diversity by sampling perturbed initial conditions $\tilde Y_0 = Y_0 + \lambda\varepsilon$,
then propagating each $\tilde Y_0$ forward according to \eqref{eq:forward_sde_from_bsde}. The parameter
$\lambda\ge 0$ controls the exploration radius around $Y_0$.

\begin{algorithm}
\caption{Conditional generation via $Y_0$-neighborhood sampling (forward Euler--Maruyama)}
\label{algo:forwardsampling_fix}
\DontPrintSemicolon
\SetKwInput{KwInput}{Input}
\SetKwInput{KwOutput}{Output}
\KwInput{Initial state $Y_0$, $Z$-estimator $\Phi(\cdot;\beta)$, steps $N$, score $s_\theta$, driver $\hat f$, scale $\lambda$}
\KwOutput{Terminal sample $\tilde Y_T$}
\BlankLine
$\Delta t \leftarrow T/N$\;
$\varepsilon_0 \sim \mathcal N(0,I_d)$\;
$y \leftarrow Y_0 + \lambda\,\varepsilon_0$\;
$W \leftarrow 0$\;
\For{$i=0$ \KwTo $N-1$}{
    $t_i \leftarrow i\Delta t$\;
    $\varepsilon_i \sim \mathcal N(0,I_d)$,\;\; $\Delta W_i \leftarrow \sqrt{\Delta t}\,\varepsilon_i$\;
    $W \leftarrow W + \Delta W_i$\;
    $Z_i \leftarrow \Phi(t_i, W;\beta)$\tcp*{$Z_i \approx Z_{t_i}\in\mathbb R^{d\times d}$}
    $b_i \leftarrow -\,\hat f\!\big(s_\theta(y,t_i),\,Z_i,\,t_i\big)$\tcp*{drift of \eqref{eq:forward_sde_from_bsde}}
    $y \leftarrow y + b_i\,\Delta t + Z_i\,\Delta W_i$\;
}
\Return{$y$}\;
\end{algorithm}

\subsection*{Algorithm~\ref{algo:UQ_fix}: Monte Carlo uncertainty quantification}

\noindent
To quantify uncertainty in conditional generation, we repeatedly run a chosen forward sampler
(e.g., Algorithm~\ref{algo:forwardsampling_fix} with
independent randomness (Brownian increments and/or $Y_0$ perturbations) to obtain an empirical distribution
of terminal samples. We report the sample mean and covariance.

\begin{algorithm}
\caption{Monte Carlo uncertainty quantification for terminal samples}
\label{algo:UQ_fix}
\DontPrintSemicolon
\SetKwInput{KwInput}{Input}
\SetKwInput{KwOutput}{Output}
\KwInput{Sampler $\mathcal A$ (Algorithm~\ref{algo:forwardsampling_fix}), number of runs $K$}
\KwOutput{Sample mean $\bar y$, sample covariance $\widehat{\mathrm{Cov}}$}
\BlankLine
\For{$k=1$ \KwTo $K$}{
    $y^{(k)} \leftarrow \mathcal A()$\tcp*{independent random seed / Brownian path}
}
$\bar y \leftarrow \frac{1}{K}\sum_{k=1}^K y^{(k)}$\;
$\widehat{\mathrm{Cov}} \leftarrow \frac{1}{K-1}\sum_{k=1}^K (y^{(k)}-\bar y)(y^{(k)}-\bar y)^\top$\;
\Return{$\bar y,\widehat{\mathrm{Cov}}$}\;
\end{algorithm}

\section{Additional Qualitative Evaluation}
\subsection{SVCT Reconstruction}
\label{app:ldctsamples}

\paragraph{Additional qualitative comparisons.}
Fig.~\ref{fig:appldct} reports additional SVCT reconstruction examples across diverse anatomical
slices. Consistent trends can be observed.
\emph{FBP} exhibits severe streaking and noise amplification, especially around high-attenuation structures, which
obscures low-contrast tissue boundaries. \emph{TV} effectively suppresses noise but introduces noticeable
over-smoothing and loss of fine textures (e.g., lung parenchyma and soft-tissue interfaces), together with bias in
local contrast. \emph{TVDIP} further reduces streak artifacts compared to TV, but still tends to blur anatomical edges
and dampen subtle structures, suggesting a trade-off between artifact removal and detail preservation.

Among diffusion-based baselines, SiTCoM, PnPDM, DPS, and DDS further illustrate different behaviors under
step-wise conditioning. \emph{SiTCoM} can recover sharper local structures than TV/TVDIP, but its intensity is less
stable across cases and it still exhibits noticeable streak residuals and display-window sensitivity. \emph{PnPDM} is
visually smoother and reduces streaking, but it often suppresses high-frequency textures and weak edges, yielding an
over-regularized appearance and attenuated structural contrast. \emph{DPS} is more sensitive to the measurement-guidance
balance in this setting and tends to produce grainy, unstable textures, especially in soft-tissue regions, indicating that
posterior-sampling guidance alone does not fully prevent drift toward visually plausible but poorly calibrated samples.
\emph{DDS} is visibly stronger than the other diffusion baselines and recovers sharper anatomical boundaries with fewer
gross failures. However, we emphasize that this result is obtained with a \emph{released CT-pretrained prior}, so it does
not correspond to exactly the same general-prior setting studied in this paper. In the present comparison, DDS should
therefore be interpreted mainly as a strong reference point rather than a like-for-like baseline.

In contrast, our method follows a terminal-conditioned formulation that is fundamentally different from step-wise
PnP-style designs. Rather than repeatedly injecting heuristic guidance along the stochastic path, we impose a hard
terminal feasibility operator at the endpoint of the reverse process. This enables composition with arbitrary pretrained
generative priors in a modular manner: the prior specifies the base stochastic dynamics, while the task information is
isolated into the terminal specification that enforces measurement-consistent decoding and projection. As a result, the
final reconstruction is explicitly tied to the measurement model through the terminal constraint, making the solution
less fragile to the absolute strength of the prior or to prior pretraining domains that are not perfectly matched to
the measurement target. When the measurement operator is accurate and the terminal projection is well-defined,
endpoint feasibility is enforced by construction up to numerical tolerances, whereas PnPDM, DPS, and SiTCoM do
not provide an explicit terminal guarantee and can drift toward prior-preferred samples when the prior and the
measurement target are inconsistent. Qualitatively, BSDE remains the most stable across slices: it suppresses streaks,
preserves subtle textures and soft-tissue transitions, and avoids the contrast drift or grainy artifacts seen in the
step-wise diffusion baselines. Compared with DDS, BSDE is slightly smoother in some regions but shows more
calibrated global appearance under the same measurement model, which is consistent with our goal of enforcing
terminal feasibility under a general-prior setting rather than relying on a tightly matched CT-domain prior.

Notably, BSDE better maintains lung texture and
sharp soft-tissue transitions without introducing the over-smoothing seen in TV/TVDIP. 

\begin{figure}[h]
    \centering
    \includegraphics[width=\linewidth]{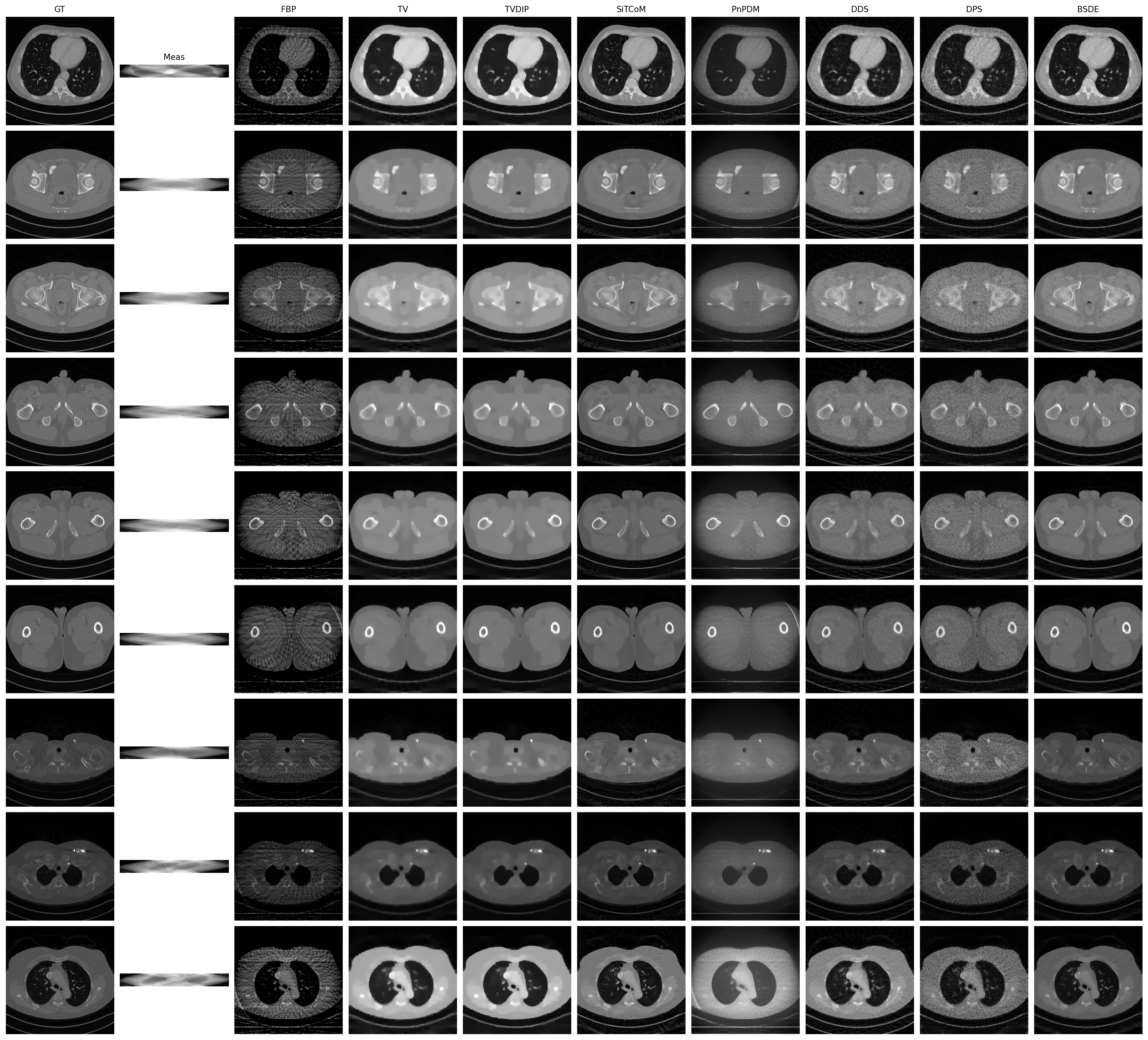}
    \caption{{Additional qualitative results.}}
    \label{fig:appldct}
\end{figure}

\subsection{Further discussion on DDS and prior matching}
Table~\ref{tab:ct_results} shows that, under the intended general-prior setting, BSDE provides the strongest overall quantitative performance.
Among the baselines that do not rely on a released CT-domain pretrained prior in our evaluation setting, BSDE outperforms PnPDM, DPS, and SiTCoM across all reported image-quality and error metrics.
This is the main quantitative comparison aligned with our problem definition: enforcing terminal feasibility under a frozen general prior, rather than assuming that the prior is already tightly matched to the CT measurement domain.

We include DDS here only as an additional reference. The reason is that DDS uses a released \emph{CT-pretrained prior}, and therefore does not correspond to the same general-prior setting studied in this paper.
For this reason, we do not center the main quantitative claim on DDS, and defer a more detailed discussion of this fairness mismatch to the Appendix.

Even under this non-like-for-like comparison, the result is still useful to read carefully.
DDS attains slightly higher PSNR/NPSNR ($31.10/30.93$ vs.\ $30.99/30.85$), but BSDE remains better on MAE ($3.57$ vs.\ $4.38$), SSIM ($0.868$ vs.\ $0.806$), NSSIM ($0.869$ vs.\ $0.811$), NMSE ($0.0101$ vs.\ $0.0118$), and NCC ($0.989$ vs.\ $0.988$).
So we do not read DDS as overturning the main result.
Rather, it indicates that a domain-matched CT prior can slightly favor pixel-wise distortion metrics, while BSDE still retains stronger structural consistency and lower overall error in the intended general-prior regime.

Therefore, our main takeaway remains unchanged:
BSDE is most competitive in the setting this paper is actually about, namely reconstruction with a frozen prior that may be imperfect or mismatched, where correctness is enforced by the terminal constraint rather than inherited from a domain-specific CT prior.

\subsection{BSDE for star Light Curve Inversion and Conditional Generation}
\label{app:toysamples}
The visual comparison of the proposed approach against other generation methods is presented in the Fig. \ref{fig:congenlight}.  We use the star lightcurve dataset \cite{Rebbapragada_2008} , which encapsulates time series data corresponding to three distinct categories of star lightcurves: Cepheid, Eclipsing Binary, and RR Lyrae. Each stellar class portrays exclusive periodic light fluctuations, contributing significantly to astronomical explorations. Fig. \ref{fig:congenlight} manifests unsupervised conditional generation results of individual samples for various methods across the triad of distinct classes.

Our proposed BSDE-based Diffusion model exhibits a noteworthy performance superiority over the GAN-inversion \cite{xia2022gan}, LSTMVAE \cite{lstmvae}, and SDE-based Diffusion\cite{meng2022sdedit} methods in the realm of unsupervised conditional generation of star lightcurves signals. Fig. \ref{fig:congenlight} presents rows of unsupervised conditional generation results sourced from the proposed BSDE-based Diffusion model, the GAN-based Inversion model, the SDE-based Diffusion Perturbing model, the LSTMVAE-based model, and samples drawn from the target classes for the purpose of comparison. From this visual inspection, it is evident that the BSDE Diffusion model continually surpasses its counterparts for all three classes of star types (Cepheid, Eclipsing Binary, and RR Lyrae), implying the model's enhanced precision in echoing the genuine star lightcurve distributions.

The visualization shows the effectiveness of our BSDE-based Diffusion model in capturing and replicating the complex temporal dynamics inherent in these celestial time series data, thus establishing it as a promising tool for advance time series generation tasks.

\begin{figure}[h]
    \centering
    \includegraphics[width=\textwidth]{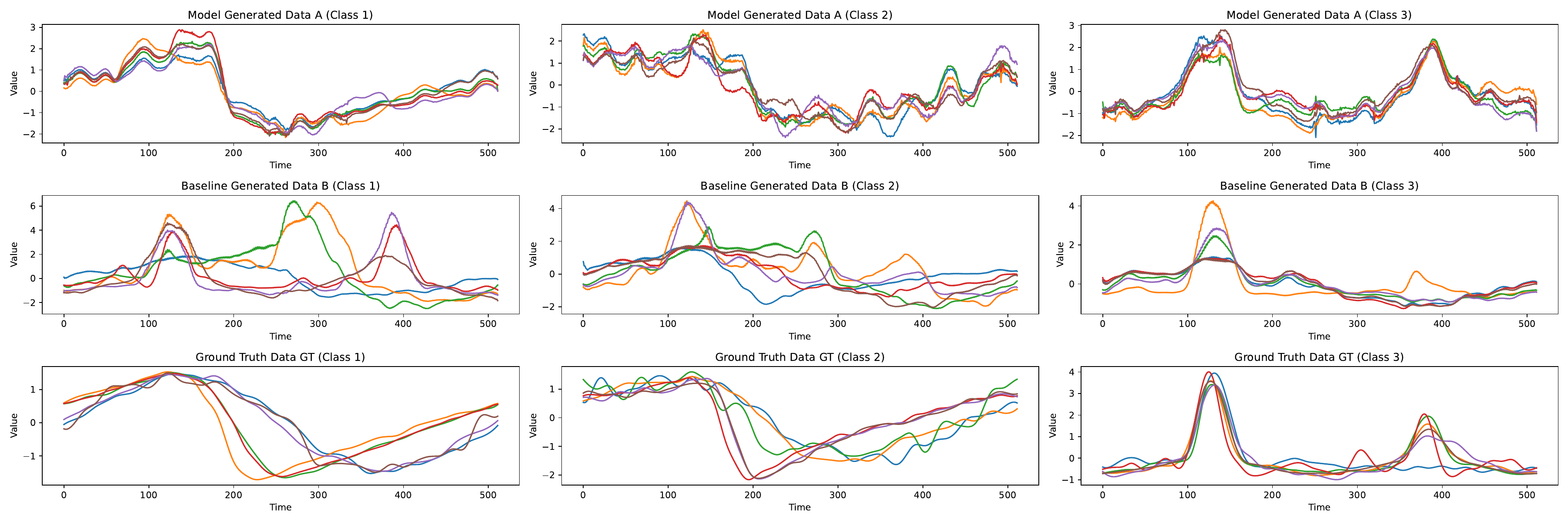}
    \caption{Conditional Generation of lightcurves (time series data).}
    \label{fig:congenlight}
\end{figure}

\begin{figure}[h]
    \centering
    \includegraphics[width=\textwidth]{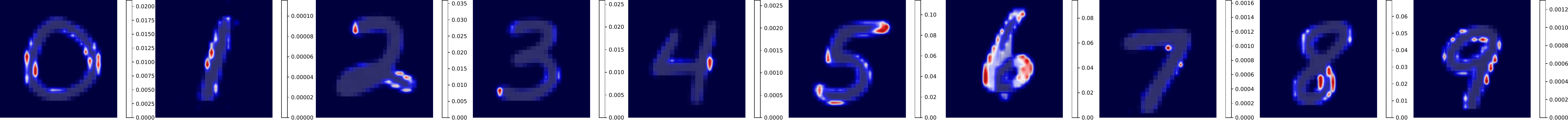}
    \caption{One sample uncertainty quantification on MNIST dataset.}
    \label{fig:uq}
\end{figure}

\begin{figure}[h]
    \centering
    \includegraphics[width=\textwidth]{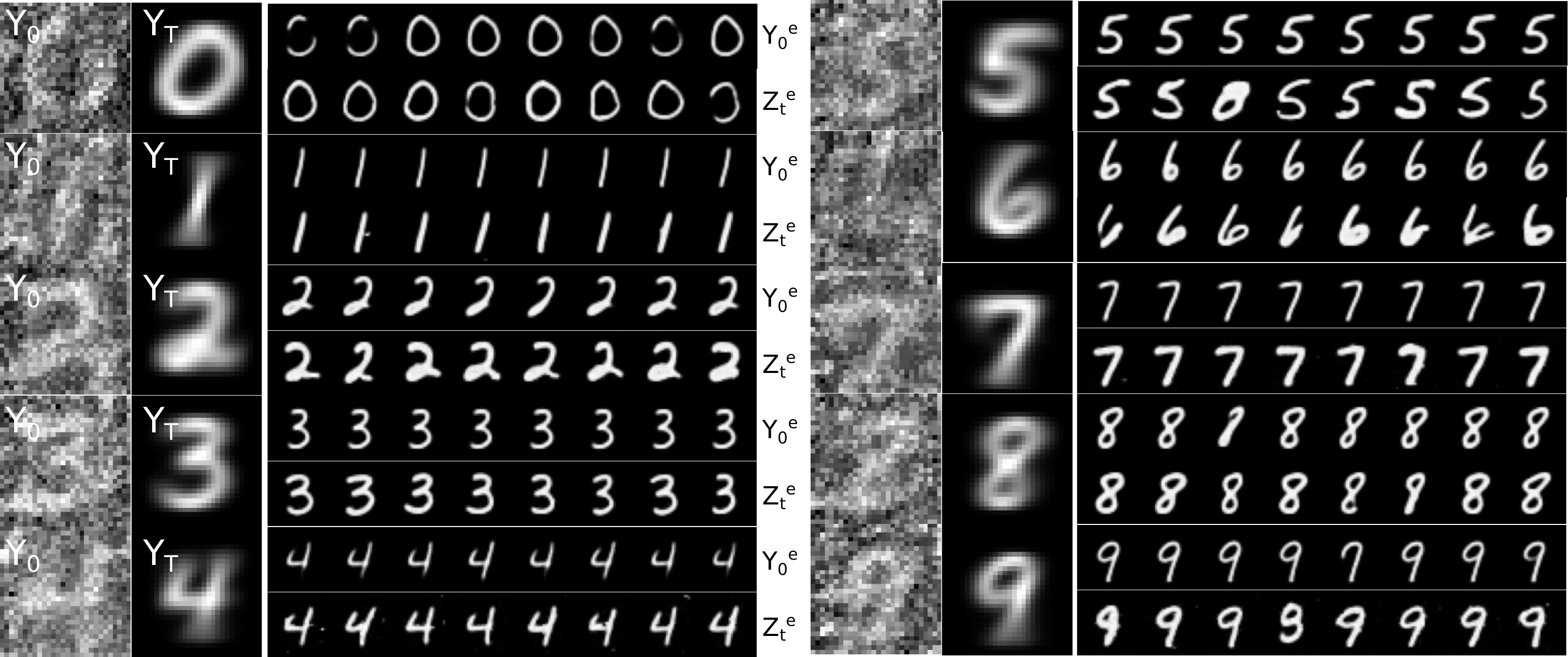}
    \caption{BSDEs-based Diffusion for Conditioning Generation on MNIST Dataset (2D data).}
    \label{fig:ResultSamplingAppendix}
\end{figure}

\subsection{BSDE for Diffusion Inversion and Conditional Generation}
Fig. \ref{fig:ResultSamplingAppendix} displays the diffusion inversion outcomes achieved when employing the BSDE-based Diffusion model with target mean images of $Y_T$ from each class to procure the latent representation $Y_0$ of the intended target.

Once we obtain the latent noise encoding, 
as shown in the Fig. \ref{fig:ResultSamplingAppendix}, we can then perform conditional generation by using the proposed \textit{$Y_0$-neighborhood Sampling} algorithm. The figure shows diversity of the generation for different numbers. In the absence of the two transformations, the uncertainty tied to a single sample generation can still be quantified via the MC method delineated in algorithm \ref{algo:UQ_fix}. We exhibit the uncertainty related to one sample generation in Fig. \ref{fig:uq}.



\section{Additional Experimental and Implementation Details}
\label{app:extra_details}

\subsection{SVCT forward model and exact experimental setting}
\label{app:svct_setup}

To make the CT setting fully explicit, we describe the degradation model used in the SVCT experiments.
The task is \emph{sparse-view CT with light noise}, not sparse-view CT in the strict Poisson-photon sense.
We use a parallel-beam Radon geometry with \textbf{30 uniformly spaced projection views over $180^\circ$},
i.e., over the nonredundant angular range $[0,\pi]$.
Let $x_0 \in \mathcal{X}$ denote the target image and let $\mathcal{A}:\mathcal{X}\to\mathcal{Y}$ denote the forward projector.
The measurement is generated as
\begin{equation}
y_0 = \mathcal{A}(x_0) + \eta,
\end{equation}
where $\eta$ is additive Gaussian noise applied to the sinogram with variance $10^{-5}$.
Thus, in our experiments, the main difficulty comes from \emph{severe angular undersampling}, while the measurement noise is relatively light.

\paragraph{Preprocessing note.}
All baselines and the proposed method are evaluated under the same forward operator $\mathcal{A}$, the same 30-view geometry, and the same Gaussian-noise level.
We also ensure consistent image-range handling during export and evaluation.

\subsection{CT-specific realization of the terminal specification}
\label{app:ct_terminal_operator}

In the toy setting, the terminal target $\xi$ is explicitly prescribed.
In CT reconstruction, however, the terminal specification is typically \emph{not} available as a closed-form latent anchor.
Instead, it is defined implicitly through measurement feasibility under the forward model.

For a measurement $y_0$, we define the feasible image set
\begin{equation}
\mathcal{S}_{\varepsilon}(y_0)
\;=\;
\left\{
x \in \mathcal{X} \;:\;
\|\mathcal{A}(x)-y_0\| \le \varepsilon \|y_0\|
\right\}.
\end{equation}
The terminal specification operator $\Psi$ should therefore be interpreted as a \emph{terminal-consistency operator}:
it encodes the requirement that the decoded terminal latent correspond to an image in $\mathcal{S}_{\varepsilon}(y_0)$.
Equivalently, in the CT case, $\Psi(y_0)$ is realized numerically by minimizing the terminal-domain measurement discrepancy
\begin{equation}
\mathcal{L}_{\mathrm{meas}}
\;=\;
\mathbb{E}\!\left[\|\mathcal{A}(\hat x)-y_0\|^2\right],
\end{equation}
where the expectation is taken over the innovation variables in the discretized score-defined recursion.
This means that the CT pipeline is not an example where the terminal anchor is removed.

\paragraph{CT-specific implicit realization of the terminal condition.}
For completeness, we summarize the CT pipeline used to realize the implicit terminal specification:
\begin{enumerate}
    \item Given a measurement $y_0$, initialize the BSDE parameters $(\alpha,\beta)$ for the recovered noisy prior state and the representation process.
    \item Solve the BSDE on $[0,\tau]$ to obtain the inferred latent state at noise level $\tau$.
    \item Run the frozen denoising recursion from noise level $\tau$ to the data end.
    \item Decode the resulting data-end latent to obtain $\hat x = D(\hat z_0)$.
    \item Evaluate the terminal-domain measurement loss
    \[
    \mathcal{L}_{\mathrm{meas}}
    =
    \mathbb{E}\!\left[\|\mathcal{A}(\hat x)-y_0\|^2\right].
    \]
    \item Update $(\alpha,\beta)$ by gradient descent through the discretized BSDE solver and denoising recursion.
    \item Repeat until convergence and output the final reconstruction $\hat x$.
\end{enumerate}

\subsection{Explicit-terminal and implicit-terminal regimes}
\label{app:explicit_implicit_terminal}

The paper contains two practically different terminal-conditioned regimes.

\paragraph{Explicit-terminal regime.}
In the toy experiments, the terminal target $\xi$ is explicitly given.
The BSDE is therefore trained by direct terminal matching, and the terminal-value formulation appears in its most direct form.

\paragraph{Implicit-terminal regime.}
In real inverse problems such as CT, a closed-form latent target is typically unavailable.
The terminal condition still exists conceptually, but it is specified through forward-model-defined feasibility rather than by an explicit latent anchor.
This is why Eq.~(22) is introduced: it numerically realizes the same terminal requirement in the terminal domain.
In this sense, the CT pipeline remains terminal-conditioned, but the terminal specification is \emph{implicit}.

This distinction is important because it clarifies why the CT implementation should not be interpreted as abandoning the BSDE formulation.
The optimization is still carried out within the BSDE-parameterized family induced by the terminal-conditioned construction, rather than as unconstrained latent fitting.

\subsection{Baseline implementation details}
\label{app:baseline_details}

To make the comparisons reproducible, we summarize the baseline implementations and the main hyperparameters used in our experiments.
We start from the released implementations of each method and only adapt the task-specific CT geometry and measurement setup to the 30-view SVCT setting with Gaussian noise variance $10^{-5}$.

\subsection{Further discussion on prior matching and why we do not require a CT-specific prior}
\label{app:prior_matching}

It is common in inverse-problem papers to use a prior pretrained on tightly matched in-domain data.
However, this is not the setting studied in our paper.
Our goal is to handle the more realistic case where the prior may be imperfect or mismatched with the target measurement domain.
In practice, mismatch can arise for many reasons, including scanner differences, protocol changes, preprocessing pipelines, site-specific variation, or domain shifts between the source data used to train the prior and the test-time measurements.

This is precisely the issue our BSDE formulation is designed to address.
We do not want correctness to depend on the prior already being well aligned with the measurement distribution.
Instead, the forward model defines the terminal requirement, and the BSDE constrains the final reconstruction to remain consistent with the observed measurement even when the prior is imperfect or mismatched.

In that sense, the main distinction is not simply whether a method uses a task-specific prior.
For example, SiTCoM also does not require a task-specific retrained prior.
The more important difference is how measurement consistency is enforced:
our method uses terminal feasibility as the main constraint, rather than relying on the prior to already be correct.

\subsection{How the score-induced driver is chosen}
\label{app:driver_choice}

A remaining practical question is how the BSDE driver $\hat f$ is chosen once a prior family is fixed.
The general associated BSDE is written using a generic driver notation to allow dependence on the forward prior process.
In the score-induced instantiation used in this paper, the effective driver is chosen as a coupling of the pretrained score field and the BSDE state:
\begin{equation}
f_{\theta}(t,x,y,z) \;=\; \hat f\bigl(t, s_{\theta}(y,t), z\bigr).
\end{equation}
The dependence on the prior family therefore enters through the selected score model $s_{\theta}$ and its noise schedule.

For the toy VE-style construction used in Eq.~(17), the driver specializes to
\begin{equation}
\hat f\!\left(s_{\theta}(Y_t,t), Z_t\right)
=
-\sigma_t^2 s_{\theta}(Y_t,t) + \alpha Z_t.
\end{equation}
Setting $\alpha=0$ isolates the effect of terminal conditioning and yields
\begin{equation}
dY_t = -\sigma_t^2 s_{\theta}(Y_t,t)\,dt + Z_t\,dW_t.
\end{equation}
Thus, Eq.~(17) is not an arbitrary choice: it is the score-induced specialization associated with the selected VE-style prior and the chosen terminal-conditioning mechanism.

\subsection{Practical limitations and computational cost}
\label{app:limitations_cost}
Our method inherits the practical limitations of neural BSDE solvers and score-based inverse solvers more broadly.

In our current implementation, a representative CT reconstruction run takes approximately $41.3$\,s for inference, plus $2.8$\,s for model loading.
The optimization stage dominates the runtime ($40.9$\,s).
Peak GPU memory usage is about $886$\,MB allocated ($898$\,MB reserved).

The theoretical guarantees in Section~3 apply to the exact adapted BSDE solution under the stated assumptions.
In practice, the method uses time discretization and neural approximation of the representation process.
A preliminary discretization study indicates that the method remains reasonably stable over several solver choices, with relative measurement residuals around $0.0015$--$0.0017$ for
$k_{\mathrm{sde}} \in \{1,8,16,32\}$.

Still, approximation and discretization errors remain part of the practical gap between theory and implementation.


\end{document}